\documentclass[11pt,twoside]{article}
\usepackage[margin=1in]{geometry}

\usepackage{amsthm}
\usepackage{url}
\usepackage{hyperref}
\usepackage{amsfonts}
\usepackage{booktabs}
\usepackage{siunitx}
\usepackage{multirow}
\usepackage{graphicx}
\usepackage{listings}
\usepackage{commath}
\usepackage{epstopdf}
\usepackage{xcolor}
\usepackage{url}
\usepackage{mathtools}
\usepackage{caption}
\usepackage{alltt}
\usepackage{mathrsfs}
\usepackage{float}
\usepackage{caption}
\usepackage[normalem]{ulem}
\usepackage{graphicx,fancyvrb}
\usepackage[colorinlistoftodos]{todonotes}
\usepackage[linesnumbered,ruled,vlined]{algorithm2e}
\usepackage{algpseudocode}
\usepackage{amsmath,amssymb}
\usepackage{enumitem}
\usepackage{booktabs}
\usepackage{caption}
\usepackage{float}
\usepackage{capt-of}
\usepackage{multirow}

\usepackage{array}
\usepackage{arydshln}
\usepackage{todonotes}

\definecolor{mygreen}{rgb}{0,0.6,0}
\definecolor{mygray}{rgb}{0.5,0.5,0.5}
\definecolor{mymauve}{rgb}{0.58,0,0.82}
\definecolor{cadmiumgreen}{rgb}{0.0, 0.42, 0.24}

\usepackage{listings}
\lstnewenvironment{VerbatimText}[1][]{
    
    \lstset{fancyvrb=true,basicstyle=\footnotesize,captionpos=b,xleftmargin=2em,#1}
}{}

\newcommand{\ignore}[1]{}

\newcommand{\batya}[1]{{\texttt{\color{blue} Batya: [{#1}]}}}

\newcommand{\fd}{\rightarrow}

\newcommand{\real}{{\mathbb{R}}}

\newcommand*{\rom}[1]{\expandafter\@slowromancap\romannumeral #1@}
\newcommand{\RNum}[1]{\uppercase\expandafter{\romannumeral #1\relax}}
\newtheorem{defn}{Definition}[section]
\newtheorem{example}[defn]{Example}
\newtheorem{corr}[defn]{Colollary}
\newtheorem{lemma}[defn]{Lemma}
\newtheorem{lem}[defn]{Lemma}
\newtheorem{thm}[defn]{Theorem}
\newtheorem{theorem}[defn]{Theorem}

\newcommand{\sel}[1]{{\sigma}}

\newcommand{\cut}[1]{}
\newcommand{\eat}[1]{}

\newcommand{\defeq}{\stackrel{\text{def}}{=}}
\def\set#1{\mathord{\{#1\}}}

\def\cl#1{\mathbf{cl}\left(#1\right)}

\def\iset{\mathrm{m}}
\def\isetc{\mathrm{m}^c}

\def\measure{\mu}
\def\imeasure{\mu^*}
\def\eqdef{\mathrel{\stackrel{\textsf{\tiny def}}{=}}}

\def\D{\mathcal{D}}
\def\e#1{\emph{#1}}

\newenvironment{citedtheoremJAIR}[1]
{\begin{thm}{\it\e{#1}}\,\,}
	{\end{thm}}
\newenvironment{citedlemmaJAIR}[1]
{\begin{lem}{\it\e{#1}}\,\,}
	{\end{lem}}
\newenvironment{citeddefnJAIR}[1]
{\begin{defn}{\it\e{#1}}\,\,}
	{\end{defn}}

\def\implies{\Rightarrow}

\newcommand{\pow}[1]{2^{{#1}}} % \mathbf{Pow}

\def\vplus{{+}}

\newenvironment{repeatresult}[2]
{\vskip0.5em\par\textsc{#1} #2.\em}
{\vskip1em}

\newenvironment{reptheorem}[1]{\begin{repeatresult}{Theorem}{#1}}{\end{repeatresult}}
\newenvironment{replemma}[1]{\begin{repeatresult}{Lemma}{#1}}{\end{repeatresult}}

\def\appendix{\par
	\section*{APPENDIX}
	\setcounter{section}{0}
	\setcounter{subsection}{0}
	\def\thesection{\Alph{section}} }

\def\entropicFunctions{\Gamma^*}
\def\entropicPlhdrl{\Gamma}

\def\rankConen{\mathcal{R}_n}

\def\monotonicConen{\mathcal{M}_n}

\def\positiveConen{\mathcal{P}_n}

\def\iset{\mathrm{m}}
\def\isetc{\mathrm{m}^c}

\def\measure{\mu}
\def\imeasure{\mu^*}
\def\eqdef{\mathrel{\stackrel{\textsf{\tiny def}}{=}}}

\def\e#1{\emph{#1}}

\def\impliedCI{\tau}
\def\Yes{\mathbb{Y}}
\def\No{\mathbb{N}}

\def\bd{\boldsymbol{d}}

\def\vars{\mathbf{var}}

\def\sigmapair{\Sigma_{\text{pair}}}
\def\sigmarb{\Sigma_{\text{RB}}}

\newenvironment{proofOverview}{\begin{proof}[Proof Overview]}{\end{proof}}
\title{Approximate Implication for Probabilistic Graphical Models}

\author{
Batya Kenig
}
\date{}
\begin{document}

% For research notes, remove the comment character in the line below.
% \researchnote

\maketitle

\begin{abstract}
The graphical structure of Probabilistic Graphical Models (PGMs) represents the conditional independence (CI) relations that hold in the modeled distribution. Every \e{separator} in the graph represents a conditional independence relation in the distribution, making them the vehicle through which new conditional independencies are inferred and verified. The notion of separation in graphs depends on whether the graph is directed (i.e., a \e{Bayesian Network}), or undirected (i.e., a \e{Markov Network}). \eat{
Graph algorithms, such as \e{d-separation}, use this structure to infer additional conditional independencies.\eat{, and to query whether a specific CI holds in the distribution.}}

The premise of all current systems-of-inference for deriving CIs in PGMs, is that the set of CIs used for the construction of the PGM hold \e{exactly}.
In practice, algorithms for extracting the structure of PGMs from data discover \e{approximate CIs} that do not hold exactly in the distribution. In this paper, we ask how the error in this set propagates to the inferred CIs\eat{, or implied , discovered by current systems of inference} read off the graphical structure. More precisely, what guarantee can we provide on the inferred CI when the set of CIs that entailed it hold only approximately? It has recently been shown that in the general case, no such guarantee can be provided. 

In this work, we prove new negative and positive results concerning this problem. We prove that separators in undirected PGMs do not necessarily represent approximate CIs. In other words, no guarantee can be provided for CIs inferred from the structure of undirected graphs.
We prove that such a guarantee exists for the set of CIs inferred in directed graphical models, making the \e{$d$-separation} algorithm a sound and complete system for inferring \e{approximate CIs}. We also establish improved approximation guarantees for independence relations derived from \e{marginal} and \e{saturated} CIs.

\end{abstract}
\section{Introduction}
Conditional independencies (CI) are assertions of the form $X\bot Y|Z$, stating that the random variables (RVs) $X$ and $Y$ are independent when conditioned on $Z$. The concept of conditional independence is at the core of Probabilistic graphical Models (PGMs) that include Bayesian and Markov networks. The CI relations between the random variables enable the modular and low-dimensional representations of high-dimensional, multivariate distributions, and\eat{enable efficient} tame the complexity of inference and learning, which would otherwise be very inefficient~\cite{DBLP:books/daglib/0023091,DBLP:books/daglib/0066829}.

The \e{implication problem} is the task of determining whether a set of CIs termed \e{antecedents} logically entail another CI, called the \e{consequent}, and it has received considerable attention from both the AI and Database communities~\cite{DBLP:conf/ecai/PearlP86,DBLP:conf/uai/GeigerVP89,DBLP:journals/iandc/GeigerPP91,SAYRAFI2008221,DBLP:conf/icdt/KenigS20,DBLP:conf/sigmod/KenigMPSS20}. Known algorithms for deriving CIs from the topological structure of the graphical model are, in fact, an instance of implication. 
The Directed Acyclic Graph (DAG) structure of Bayesian Networks is generated based on a set of CIs termed the \e{recursive basis}~\cite{DBLP:journals/networks/GeigerVP90}, and the $d$-separation algorithm is used to derive additional CIs, implied by this set. In undirected PGMs, also called Markov networks or Markov Random Fields (MRFs), every pair of non-adjacent vertices $u$ and $v$ signify that $u$ and $v$ are conditionally independent given the rest of the vertices in the graph. If the underlying distribution is strictly positive, then this set of CIs (i.e., corresponding to the pairs of non-adjacent vertices) imply a much larger set of CIs associated with the \e{separators} of the graph~\cite{StudenyBookChapter}. 
A separator in an undirected graph $G(V,E)$ is a subset of vertices $C\subseteq V$ whose removal breaks the graph into two or more connected components. Specifically, if $C$ is a separator in the undirected PGM $G(V,E)$, then the vertex set $V{\setminus}C$ can be partitioned into two disjoint sets $A,B\subseteq V$, where every path between a vertex $a\in A$ and $b\in B$ passes through a vertex in $C$. This partitioning corresponds to the conditional independence relation $A\bot B|C$. 

The $d$-separation algorithm is a sound and complete method for deriving CIs in probability distributions represented by DAGs~\cite{DBLP:conf/uai/GeigerVP89,DBLP:journals/networks/GeigerVP90}. In undirected PGMs, graph-separation completely characterizes the conditional independence statements that can be derived from the simple conditional independence statements associated with the non-adjacent vertex pairs of the graph~\cite{DBLP:journals/kybernetika/PearlGV89,GeigerPearl1993,StudenyBookChapter}.
The foundation of deriving CIs in directed and undirected models is the \e{semigraphoid axioms} and the \e{graphoid axioms}, respectively~\cite{Dawid1979,GEIGER1991128,GeigerPearl1993}.
\eat{
	More precisely, the soundness of $d$-separation was established by Verma, who also showed that the set of CIs derived by $d$-separation is precisely the closure of the recursive basis under the \e{semgraphoid axioms}~\cite{GEIGER1991128,GeigerPearl1993}. Then, Geiger and Pearl proved that the semigraphoid axioms are a complete system-of-inference for deriving CIs from the recursive basis, thus showing that $d$-separation is complete. }

Current systems for inferring CIs, and the graphoid axioms in particular, assume that both antecedents and consequent hold \e{exactly}, hence we refer to these as an exact implication (EI). 
However, almost all known approaches for learning the structure of a PGM rely on CIs extracted from data, which hold to a large degree, but cannot be expected to hold exactly. Of these, structure-learning approaches based on information theory have been shown to be particularly successful, and thus widely used to infer networks in many fields~\cite{CHENG200243,JMLR:v7:decampos06a,Chen2008TKDE,Zhao5130,DBLP:conf/sigmod/KenigMPSS20}. 

In this paper, we drop the assumption that the CIs hold exactly, and
consider the \e{relaxation
	problem}: if an exact implication holds, does an \e{approximate implication} hold too? 
That is, if the antecedents approximately hold in the distribution, does the consequent approximately hold as well?
What guarantees can we give for the approximation? In other words, the relaxation problem
asks whether, and under what conditions, we can convert an exact implication to an approximate one. 
When relaxation holds, then the error
to the consequent can be bounded, and any system-of-inference for deriving exact implications (e.g., the semigraphoid axioms, $d$-separation, graph-separation), can be used to infer an approximate implication.

To study the relaxation problem we need to measure the degree of satisfaction of
a CI. In line with previous work, we use Information Theory. This is the natural semantics for
modeling CIs because $X \bot Y | Z$ if and only if $I(X; Y |Z) = 0$, where $I$ is the
conditional mutual information. Hence, an exact implication (EI) $\sigma_1,\cdots,\sigma_k \implies \tau$ is an assertion of the form $(h(\sigma_1){=}0 \wedge  \cdots \wedge h(\sigma_k){=}0) \implies
h(\tau){=}0$, where 
$\tau,\sigma_1,\sigma_2,\dots$ are triples $(X;Y|Z)$, and
$h$ is the conditional mutual information measure $I(\cdot;\cdot|\cdot)$. An approximate implication (AI) is a linear inequality $h(\tau) \leq \lambda h(\Sigma)$, where $h(\Sigma)\eqdef\sum_{i=1}^kh(\sigma_i)$, and $\lambda \geq 0$ is the approximation factor. We say that a class of CIs \e{$\lambda$-relaxes} if every exact implication (EI) from the class can be transformed to an approximate implication (AI) with an approximation factor $\lambda$. We observe that an approximate implication always implies an exact implication\eat{, even if we do not have an upper bound on $\lambda$, but every concrete AI has a finite lambda This is} because the mutual information $I(\cdot;\cdot|\cdot)\geq 0$ is a nonnegative measure. Therefore, if $0\leq h(\tau)\leq \lambda h(\Sigma)$ for some $\lambda \geq 0$, then $h(\Sigma)=0 \implies h(\tau)=0$.
\eat{\footnote{This is because $I(\cdot;\cdot|\cdot)\geq 0$ is a nonnegative measure. Therefore, if $0\leq h(\tau)\leq \lambda h(\Sigma)$, then $h(\Sigma)=0 \implies h(\tau)=0$. }. }

\textbf{Results.} 
A conditional independence assertion $(A;B|C)$ is called \e{saturated} if it mentions all of the random variables in the joint distribution, and it is called \e{marginal} if $C=\emptyset$. The exact variant of implication was extensively studied ~\cite{DBLP:conf/uai/GeigerVP89,GeigerPearl1993,DBLP:conf/uai/GeigerP88,DBLP:journals/iandc/GeigerPP91,DBLP:journals/networks/GeigerVP90} (see below the related work). In this paper, we study approximate implication. Our results are summarized in Table~\ref{tab:AIResults}.

We first consider exact implications $\Sigma \implies \tau$, where the set of antecedents $\Sigma$ is comprised of saturated CIs, and no assumption is made on the consequent CI $\tau$. Every separator $Z$ in an undirected PGM $G(V,E)$ corresponds to a saturated CI statement $(X;Y|Z)$, where every path from a vertex $x\in X$ to a vertex $y\in Y$ passes through a vertex in $Z$. Hence, $\Sigma$ can be viewed as a set of CIs that hold in a probability distribution represented by a Markov Network.
We show that if $\tau$ can be derived from $\Sigma$ by applying the \e{semigraphoid axioms}, then the implication relaxes. Specifically, if $\tau=(A;B|C)$, then $h(\tau)\leq \min\set{|A|,|B|}h(\Sigma)$ (i.e., where $|A|$ denotes the number of RVs in the set $A$). For $n$ jointly-distributed random variables, this leads to a relaxation bound of $\min\set{|A|,|B|}\leq \frac{n}{2}$. In previous work, it was shown that $h(\tau)\leq |A|{\cdot}|B|{\cdot}h(\Sigma)$, leading to a relaxation bound of $\frac{n^2}{4}$~\cite{DBLP:journals/lmcs/KenigS22}. This work (Theorem~\ref{thm:saturated}) tightens the relaxation bound by an order of magnitude.

We also prove a negative result. If the implication involves the application of the \e{intersection axiom} (i.e., that is one of the \e{graphoid axioms}~\cite{DBLP:books/daglib/0066829}), then no relaxation exists. We present a strictly positive probability distribution in which the intersection axiom does not relax. Consequently, no relaxation exists for implications involving the the application of the intersection axiom, or more broadly, the graphoid axioms. Inferring CIs in Markov Networks relies on the intersection axiom~\cite{DBLP:books/daglib/0066829,StudenyBookChapter}. This negative result essentially establishes that if the CI relations associated with the non-adjacent vertex pairs in the graph do not hold exactly, then no guarantee can be made regarding the CI relations associated with the separators of the graph.

We show that every conditional independence relation $(A;B|C)$ read off a Directed Acyclic Graph (DAG) by the $d$-separation algorithm~\cite{DBLP:conf/uai/GeigerVP89}, admits a $1$-approximation (Theorem~\ref{thm:recursiveCIs}). In other words, if $\Sigma$ is the \e{recursive basis} of CIs used to build the Bayesian network~\cite{DBLP:conf/uai/GeigerVP89}, then it is guaranteed that $I(A;B|C)\leq \sum_{i\in \Sigma}h(\sigma_i)$. Furthermore, we present a family of implications for which our $1$-approximation is tight (i.e., $I(A;B|C)=\sum_{i\in \Sigma}h(\sigma_i)$). This result first appeared in~\cite{DBLP:conf/uai/Kenig21}. In this paper, we simplify the proof, and relate it to the $d$-separation algorithm.

We prove that every CI $(A;B|C)$ implied by a set of marginal CIs admits a $\min\set{|A|,|B|}$-approximation. For $n$ jointly-distributed random variables, this leads to a relaxation bound of $\min\set{|A|,|B|}\leq \frac{n}{2}$. In previous work, it was shown that $h(\tau)\leq |A|{\cdot}|B|{\cdot}h(\Sigma)$, leading to a relaxation bound of $\frac{n^2}{4}$~\cite{DBLP:conf/uai/Kenig21}. The relaxation bound established in this work (Theorem~\ref{thm:marginal}) is smaller by an order of magnitude.

Of independent interest is the technique used for proving the approximation guarantees. The \e{I-measure}~\cite{DBLP:journals/tit/Yeung91} is a theory which establishes a one-to-one correspondence between information theoretic measures such as entropy and mutual information (defined in Section~\ref{sec:notations}) and set theory. Ours is the first to apply this technique to the study of CI implication. 

\textbf{Related Work.}
The AI community has extensively studied
the exact implication problem for Conditional Independencies (CI).
In a series of papers, Geiger et al. showed that the \e{semi-graphoid axioms}~\cite{DBLP:conf/ecai/PearlP86} are sound and complete for deriving CI statements that are implied by marginal CIs~\cite{GeigerPearl1993}, and \e{recursive CIs} that are used in Bayesian networks~\cite{DBLP:journals/networks/GeigerVP90,DBLP:conf/uai/GeigerP88}.
In the same paper, they also showed that, when restricted to the set of strictly positive probability distributions, the \e{graphoid axioms} are sound and complete for deriving CI statements from saturated CIs~\cite{GeigerPearl1993}.
The completeness of $d$-separation follows from the fact that the set of CIs derived by $d$-separation is precisely the closure of the recursive basis under the \e{semgraphoid axioms}~\cite{VERMA199069}. Studen{\'{y}}
proved that in the general case, when no assumptions are made on the antecendents, no finite axiomatization exists~\cite{StudenyCINoCharacterization1990}. That is, there does not exist a finite set of axioms (deductive rules) from which all general conditional independence implications can be deduced. 

The database community has studied the implication problem for integrity constraints~\cite{DBLP:journals/tods/ArmstrongD80,DBLP:conf/sigmod/BeeriFH77,10.1007/978-3-642-39992-3_17,Maier:1983:TRD:1097039},
and showed that the implication problem is decidable and axiomatizable when the antecedents are Functional
Dependencies (FDs) or \e{Multivalued Dependencies} (which correspond to saturated CIs, see~\cite{DBLP:journals/tse/Lee87,DBLP:conf/icdt/KenigS20}), and undecidable
for \e{Embedded Multivalued Dependencies}~\cite{10.1006/inco.1995.1148}. 

The relaxation problem was first studied by Kenig and Suciu in the context of database dependencies~\cite{DBLP:conf/icdt/KenigS20}, where they showed that CIs derived from a set of saturated antecedents, admit an approximate implication. Importantly, they also showed that not all exact implications relax, and presented a family of 4-variable distributions along with an exact implication that does not admit an approximation (see Theorem 16 in~\cite{DBLP:conf/icdt/KenigS20}).  Consequently, it is not straightforward that exact implication necessarily imply its approximation counterpart, and arriving at meaningful approximation guarantees requires making certain assumptions on the antecedents, consequent, derivation rules, or combination thereof.

\textbf{Organization.}
We start in Section~\ref{sec:notations} with preliminaries. In Section~\ref{sec:exactInPGMs} we describe the role exact implication has played in probabilistic graphical models, and introduce the notion of approximate implication.
We formally state the results, and their practical implications in Section~\ref{sec:results}. We prove that the intersection axiom does not relax in Section~\ref{sec:intersectionAxiomDoesNotRelax}. 
In Section~\ref{sec:lemmas} we establish, through a series of lemmas, properties of exact implication that will be used for proving our results. 
In Sections~\ref{sec:saturatedCIs},~\ref{sec:recursiveProof}, and~\ref{sec:marginalProof}, we prove the relaxation bounds for exact implications from the set of saturated CIs, the recursive basis, and marginal CIs respectively, (see Table~\ref{tab:AIResults}).
We conclude in Section~\ref{sec:conclusion}.

\begin{table*}[t]
	\parbox{1.0\linewidth}{
		\centering
		\begin{tabular}{|c|c|c|c|c|}
			\hline
			\multirow{3}{*}{Type of EI}&\multicolumn{4}{c}{Relaxation Bounds}\vline\\
			\cline{2-5}
			&\multirow{2}{*}{General} &{Saturated+FDs}&{Recursive Basis}&Marginals\\		
			&&$\implies$ any&$\implies$ any&$\implies$ any\\
			\hline
			\hline			
			Semigraphoid&$(2^n)!$~\cite{DBLP:journals/lmcs/KenigS22}&$\frac{n}{2}$~(Thm.~\ref{thm:saturated})&$1$~(Thm.~\ref{thm:recursiveCIs})&$\frac{n}{2}$~(Thm.~\ref{thm:marginal})\\
			\hline 		
			Graphoid&\multicolumn{4}{c}{$\infty$~~(Thm.~\ref{thm:intersection})}\vline\\	
			\hline		
		\end{tabular}
		\vspace{0.2cm}
		\caption{Summary of results: relaxation bounds for the implication $\Sigma \implies \tau$ under various restrictions.
		%	for the sub-cones of $\entropicPlhdrl_n$  under various restrictions.
                % R: punctuation consistency
                (1) \textit{General}; derivation rules are the semi-graphoid axioms, and no restrictions are placed on $\Sigma$. (2) $\Sigma$ is a set of saturated CIs and conditional entropies, and  $\impliedCI$ is any CI. (3) $\Sigma$ is the \e{recursive basis} used to generate the Bayesian network, and $\impliedCI$ is any CI. (4) $\Sigma$ is a set of marginal CIs, and $\impliedCI$ is any CI. (5) When the set of derivation rules include the \e{intersection axiom} (e.g., the \e{graphoid axioms}), then no finite relaxation bound exists.
           }
		\label{tab:AIResults}
	}
\end{table*}

\section{Preliminaries}
\label{sec:notations}

We denote by $[n] = \set{1,2,\ldots,n}$.  If
$\Omega=\set{X_1,\ldots, X_n}$ denotes a set of variables and
$U, V \subseteq \Omega$, then we abbreviate the union $U \cup V$ with
$UV$. 

\subsection{Conditional Independence}
Recall that two discrete random variables $X, Y$ are called {\em
	independent} if $p(X=x, Y=y) = p(X=x)\cdot p(Y=y)$ for all outcomes
$x,y$. Fix $\Omega=\set{X_1,\dots,X_n}$, a set of $n$ jointly
distributed discrete random variables with finite domains
$\D_1,\dots,\D_n$, respectively; let $p$ be the probability mass.  For
$\alpha \subseteq [n]$, denote by $X_\alpha$ the joint random variable
$(X_i: i \in \alpha)$ with domain
$\D_\alpha \defeq \prod_{i \in \alpha} D_i$.  We write
$p \models X_\beta \perp X_\gamma | X_\alpha$ when $X_\beta, X_\gamma$
are conditionally independent given $X_\alpha$; in the special case
that $X_\alpha$ functionally determines $X_\beta$, we write $p \models X_\alpha \fd X_\beta$. We say that a set of random variables $\set{X_1,\dots,X_k}$ are \e{mutually independent given $Z$} if $p(X_1=x_1,\dots, X_k=x_k|Z)=p(X_1=x_1|Z)\cdots p(X_k=x_k|Z)$. If $Z=\emptyset$, then we say that $\set{X_1,\dots,X_k}$ are \e{mutually independent}.

An assertion $X {\perp} Y | Z$ is called a {\em Conditional
	Independence} statement, or a CI; this includes $Z \rightarrow Y$ as
a special case.  When $XYZ=\Omega$ we call it \e{saturated}, and when $Z=\emptyset$ we call it \e{marginal}.  A set of
CIs $\Sigma$ {\em implies} a CI $\tau$, in notation
$\Sigma \Rightarrow \tau$, if every probability distribution that
satisfies $\Sigma$ also satisfies $\tau$.

\subsection{Background on Information Theory}
\label{subsec:information:theory}

We adopt required notation from the literature on information
theory~\cite{Yeung:2008:ITN:1457455}.  For $n > 0$, we
identify the functions
$\pow{[n]}\rightarrow \real$ with the vectors in $\real^{2^n}$.

\noindent {\bf Polymatroids.} A function\eat{\footnote{Most authors
		consider rather the space $\real^{2^n-1}$, by dropping
		$h(\emptyset)$ because it is always $0$.}} $h \in \real^{2^n}$ is
called a \emph{polymatroid} if $h(\emptyset)=0$ and satisfies the
following inequalities, called {\em Shannon inequalities}:
\begin{enumerate}
	\item Monotonicity: $h(A)\leq h(B)$ for $A \subseteq B$.
	\item Submodularity: $h(A\cup B)+h(A\cap B)\leq h(A) + h(B)$ for all $A,B \subseteq [n]$.
\end{enumerate}
The set of polymatroids is denoted $\Gamma_n \subseteq \real^{2^n}$.  For any polymatroid $h$ and subsets $A,B,C,D \subseteq [n]$, we define\footnote{Recall that $AB$
	denotes $A \cup B$.}
\begin{align}
	h(B|A) \eqdef &~h(AB) - h(A) \label{eq:h:cond} \\
	I_h(B;C|A) \eqdef &~h(AB) + h(AC) - h(ABC) - h(A) \label{eq:h:mutual:information}
\end{align}

Then, $\forall h\in \Gamma_n$, $I_h(B;C|A) \geq 0$ by submodularity, and
$h(B|A) \geq 0$ by monotonicity. We say that $A$ \e{functionally determines} $B$, in notation $A \fd B$ if $h(B|A)=0$. The \e{chain rule} is the identity:
\begin{equation} \label{eq:ChainRuleMI}
	I_h(B;CD|A)=I_h(B;C|A)+I_h(B;D|AC)
\end{equation}
We call the triple $(B;C|A)$ \e{elemental}
if $|B|=|C|=1$; $h(B|A)$ is a special case of $I_h$, because
$h(B|A) = I_h(B;B|A)$. By the chain rule, it follows that every CI $(B;C|A)$ can be written as a sum of at most $|B|\cdot|C|\leq \frac{n^2}{4}$ elemental CIs.

\noindent {\bf Entropy.} If $X$ is a random variable with
a finite domain $\D$ and probability mass $p$, then $H(X)$ denotes its
entropy
\begin{equation}\label{eq:entropy}
	H(X)\eqdef\sum_{x\in \D}p(x)\log\frac{1}{p(x)}
\end{equation}
For a set of jointly distributed random variables
$\Omega=\set{X_1,\dots,X_n}$ we define the function
$h : \pow{[n]} \rightarrow \real$ as $h(\alpha) \eqdef H(X_\alpha)$;
$h$ is called an \e{entropic function}, or, with some abuse, an
\e{entropy}. It is easily verified that the entropy $H$ satisfies the Shannon inequalities, and is thus a polymatroid. 
The quantities $h(B|A)$ and $I_h(B;C|A)$ are called the \e{conditional
	entropy} and \e{conditional mutual information} respectively.  The
conditional independence $p \models B \perp C \mid A$ holds iff
$I_h(B;C|A)=0$, and similarly $p \models A \fd B$ iff $h(B|A)=0$,
thus, entropy provides us with an alternative characterization of
CIs.

The following identity holds for conditional entropy~\cite{Yeung:2008:ITN:1457455}:
\begin{equation}
	\label{eq:chainRuleH}
	h(X_1\cdots X_k|Z)=h(X_1|Z) + h(X_2|X_1Z)+\cdots + h(X_k|X_1\cdots X_{k-1}Z)
\end{equation}
If the RVs $X_1,\dots,X_k$ are mutually independent given $Z$ then:
\begin{equation}
	\label{eq:sumMutuallyIndependent}
	h(X_1\cdots X_k|Z)=h(X_1|Z) + h(X_2|Z)+\cdots + h(X_k|Z)
\end{equation}

\noindent {\bf Axioms for Conditional Independence.} 
A dependency model $M$ is a subset of triplets $(X;Y|Z)$ for which the CI $X\bot Y|Z$ holds. For example, we say that $M$ is a dependency model of a joint-distribution $p$ if the CI $X\bot Y|Z$ holds in $p$ for every triple $(X;Y|Z)\in M$. A \e{semi-graphoid} is a dependency model that is closed under the following four axioms:
\begin{enumerate}[noitemsep]
	\item Symmetry: $X\bot Y|Z$ $\Leftrightarrow$ $X\bot Y|Z$.
	\item Decomposition: $X\bot YW |Z$ $\Rightarrow$ $X\bot Y|Z$.
	\item Weak Union: $X\bot YW |Z$ $\Rightarrow$ $X\bot Y|WZ$.
	\item Contraction: $X\bot Y|Z$ and $X\bot W|YW$ $\Rightarrow$ $X\bot YW |Z$.
\end{enumerate}
If, in addition, the dependency model is closed under the \e{intersection axiom}:
\begin{align}
	\label{eq:intersection}
	X \bot Y | ZW \text{ and } X \bot W | ZY \Rightarrow X \bot YW | Z
\end{align}
then the dependency model is called a \e{graphoid}. The \e{closure} of a set of CIs $\Sigma$ with respect to the semi-graphoid (graphoid) axioms, is a set of CIs $\Sigma' \supseteq \Sigma$ that can be derived from $\Sigma$ by repeated application of the semi-graphoid (graphoid) axioms.

Let $p$ be a joint probability distribution over the random variables $\Omega= \set{X_1,\dots,X_n}$, and let $h$ be its entropy function. Since $p\models X\bot Y|Z$ iff $I_h(X;Y|Z)=0$, and since $I_h(X;Y|Z)\geq 0$, it follows that the semi-graphoid axioms are corollaries of the chain rule (see~\eqref{eq:ChainRuleMI}). Since the chain rule, and the non-negativity of $I_f(X;Y|Z)$ hold for all polymatroids $f\in \Gamma_n$, it follows that the semi-graphoid axioms hold for all polymatroids $f\in \Gamma_n$. In fact, the semi-graphoid axioms can be generalized to the following information inequalities:
\begin{enumerate}[noitemsep]
	\item Symmetry: $I_h(X; Y|Z)=I_h(Y;X|Z)$.
	\item Decomposition: $I_h(X; YW |Z) \geq I_h(X;Y|Z)$.
	\item Weak Union: $I_h(X; YW |Z)\geq I_h(X;Y|WZ)$.
	\item Contraction: $I_h(X; Y|Z)+I_h(X;W|YZ)=I_h(X;YW|Z)$.
\end{enumerate}
This is not the case for the intersection axiom (see~\eqref{eq:intersection}), which holds only for a strict subset of the probability distributions (i.e., strictly positive distributions). In Section~\ref{sec:intersectionAxiomDoesNotRelax}, we show that contrary to the semi-graphoid axioms, the intersection axiom does not correspond to any information inequality.

\section{Exact Implication for Probabilistic Graphical Models}
\label{sec:exactInPGMs}
In this section, we formally define the notions of exact and approximate implication, and their role in undirected and directed PGMs. This provides the appropriate context for which to present our results on approximate implication in later sections. 
We fix a set of variables $\Omega = \set{X_1, \ldots, X_n}$, and consider triples of
the form $\sigma = (X;Y|Z)$, where $X,Y,Z \subseteq \Omega$, which we
call a \e{conditional independence}, CI.  An \e{implication} is a
formula $\Sigma \implies \impliedCI$, where $\Sigma$ is a set of CIs
called \e{antecedents} and $\tau$ is a CI called \e{consequent}.  For an $n$-dimensional polymatroid $h\in \Gamma_n$, and 
a CI $\sigma=(X;Y|Z)$, we define $h(\sigma)\eqdef I_h(Y;Z|X)$ (see~\eqref{eq:h:mutual:information}), for a
set of CIs $\Sigma$, we define $h(\Sigma)\eqdef\sum_{\sigma \in \Sigma}h(\sigma)$.  Fix a set $K$
s.t. $K \subseteq \Gamma_n$.
\begin{defn} \label{def:ei} The \e{exact implication} (EI)
	$\Sigma \implies \impliedCI$ holds in $K$, denoted
	$K \models_{EI} \Sigma \implies \impliedCI$ if, forall $h \in K$,
	$h(\Sigma)=0$ implies $h(\tau)=0$.  The \e{$\lambda$-approximate
		implication} ($\lambda$-AI) holds in $K$, in notation
	$K \models\lambda\cdot h(\Sigma) \geq h(\tau)$, if
	$\forall h \in K$, $\lambda\cdot h(\Sigma) \geq h(\tau)$. The
	\e{approximate implication} holds, in notation
	$K \models_{AI} (\Sigma \implies \impliedCI)$, if there exist a finite
	$\lambda \geq 0$ such that the $\lambda$-AI holds.	
\end{defn}
Notice that both exact (EI) and approximate (AI) implications
are preserved under subsets of $K$: if
$K_1 \subseteq K_2$ and $K_2 \models_x \Sigma \implies \impliedCI$, then
$K_1 \models_x \Sigma \implies \impliedCI$, for $x \in \set{\text{EI,AI}}$.

Approximate implication always implies its exact counterpart.  Indeed, if $h(\tau) \leq \lambda \cdot h(\Sigma)$
and $h(\Sigma)=0$, then $h(\tau)\leq 0$, which further implies that
$h(\tau)=0$, because $h(\tau)\geq 0$ for every triple $\tau$, and every
polymatroid $h$.  In this paper we study the reverse.

\begin{defn}
	\label{def:}
	Let $\mathcal{L}$ be a syntactically-defined class of implication
	statements $(\Sigma \Rightarrow \tau)$, and let
	$K \subseteq \Gamma_n$.  We say that $\mathcal{L}$ \e{admits a
		$\lambda$-relaxation} in $K$, if every exact implication statement
	$(\Sigma \Rightarrow \impliedCI)$ in $\mathcal{L}$ has a $\lambda$-approximation:
	$$K\models_{EI} \Sigma \Rightarrow \impliedCI \text{ if and only if }
	K \models \lambda \cdot h(\Sigma) \geq h(\impliedCI).$$\eat{  We say that
		$\mathcal{I}$ admits a \e{$\lambda$-relaxation} if every EI admits a
		$\lambda$-AI.}
\end{defn}
\begin{example} 
	Let $\Sigma{=} \set{(A;B|\emptyset),(A;C|B)}$,~and
	$\impliedCI{=}(A;C|\emptyset)$. 
	Since $I_h(A;C|\emptyset) {\leq}I_h(A;BC)$, and since $I_h(A;BC){=}I_h(A;B|\emptyset) {+} I_h(A;C|B)$ by the chain rule~\eqref{eq:ChainRuleMI}, then the exact implication $\entropicPlhdrl_n \models_{EI}\Sigma\implies \tau$ admits an AI with $\lambda=1$ (i.e., a $1$-$AI$).
	\eat{
		Then both EI and AI hold forall
		polymatroids:
		\begin{align*}
			& \text{EI} && I_h(A;B|\emptyset)=0\wedge I_h(A;C|B)=0 \Rightarrow I_h(A;C|\emptyset)=0\\
			& \text{AI} && I_h(A;C|\emptyset) \leq I_h(A;B|\emptyset) + I_h(A;C|B)
		\end{align*}
	}
\end{example}
In this paper, we focus on $\lambda$-relaxation in different subsets of $\Gamma_n$, and three syntactically-defined classes: 1) Where $\Sigma$ is a set of saturated CIs (Section~\ref{sec:MarkovNets}), 2) Where $\Sigma$ is the recursive basis of a Bayesian network (Section~\ref{sec:BNs}), and 3) Where $\Sigma$ is a set of marginal CIs.

\subsection{Markov Networks and Saturated Independence}
\label{sec:MarkovNets}
Recall that a CI $(X;Y|Z)$ is saturated if $XYZ=\Omega$.
\begin{citedtheoremJAIR}{\cite{GeigerPearl1993,DBLP:conf/sigmod/BeeriFH77}}
	\label{thm:saturatedPolymatroids}
	Let $\Sigma$ be a set of saturated CIs over the set $\Omega\eqdef \set{X_1,\dots,X_n}$ of variables, and let $\Sigma^+$ denote the closure of $\Sigma$ with respect to the semi-graphoid axioms. Let $\tau$ be a saturated CI over $\Omega$. 
	\begin{align*}
		\Gamma_n \models_{EI} \Sigma \implies \tau && \text{ if and only if } && \tau \in \Sigma^+
	\end{align*}	
\end{citedtheoremJAIR}
Theorem~\ref{thm:saturatedPolymatroids} establishes that the semi-graphoid axioms are sound and complete for inferring saturated CIs from a set of saturated CIs. \cite{DBLP:journals/ipl/GyssensNG14} improve this result by dropping any restrictions on the consequent $\tau$.
In Section~\ref{sec:saturatedCIs}, we prove that if $\Sigma$ is a set of saturated CIs, then the exact implication $\Gamma_n \models_{EI} \Sigma \implies \tau$ has an $\frac{n}{2}$-relaxation for any implied CI $\tau$. In other words, $\Gamma_n \models \frac{n}{2}{\cdot}h(\Sigma)\geq h(\tau)$.

When restricted to polymatroids that are also graphoids, then CI relations can be represented by an undirected graph.
Let $G(V,E)$ be an undirected graph, and let $u,v\in V$. We say that $u$ and $v$ are \e{adjacent} if $(u,v)\in E$. A \e{path} $t=(v_1,\dots,v_n)$ is a sequence of vertices $(v_1,\dots,v_n)$ such that $(v_i,v_{i+1})\in E$ for every $i\in \set{1,\dots,n-1}$. We say that $u$ and $v$ are \e{connected} if there is a path $(u=v_1,\dots,v_n=v)$ starting at $u$ and ending at $v$; otherwise, we say that $u$ and $v$ are \e{disconnected}. Let $X,Y \subseteq V$ be disjoint sets of vertices. We say that $X$ and $Y$ are disconnected if $x$ and $y$ are disconnected for every $x\in X$ and $y\in Y$.
Let $V'\subseteq V$. The graph \e{induced by $V'$} denoted $G[V']$ is the graph $G'(V',E')$ where $E'\eqdef \set{(u,v)\in E \mid u,v \in V'}$.  
We say that $Z\subseteq V$ is an \e{$XY$-separator} if, in the graph $G[V{\setminus}Z]$, that results from $G$ by removing the vertex-set $Z$ and the edges adjacent to $Z$, $X$ and $Y$ are disconnected. For an undirected graph $G(V,E)$, we denote by $X\bot_G Y|Z$ the fact that $Z$ is an $XY$-separator.
%A subset $S\subseteq V$ is a \e{separator} if there exists a pair of vertices $u,v\in V{\setminus}S$ such that $S$ is a $uv$-separator.

Let $p$ be a joint probability distribution over the random variables $\Omega = \set{X_1,\dots,X_n}$. We define $\sigmapair\eqdef \set{(u;v|V{\setminus}uv): p\models u\bot v|V{\setminus}uv}$. That is $\sigmapair$ is the set of vertex-pairs that are conditionally independent given the rest of the variables. Observe that every CI in $\sigmapair$ is, by definition, a saturated CI.
We define the \e{independence graph} $G(\Omega, E)$, where $E\eqdef \set{(u,v) : (u;v|V{\setminus}uv) \notin \sigmapair}$. That is, the edges of the independence graph are between vertices that are not independent given the rest of the variables. 
\begin{citedtheoremJAIR}{\cite{GeigerPearl1993}}
	\label{thm:saturatedGraphoids}
	Let $p$ be a joint probability distribution over the random variables $\Omega = \set{X_1,\dots,X_n}$, with entropy function $h_p$. Let $G(\Omega,E)$ be the independence graph generated from $\sigmapair\eqdef \set{(u;v|V{\setminus}uv): p\models u\bot v|V{\setminus}uv}$. Let $\sigmapair^+$ denote the closure of $\sigmapair$ with respect to the graphoid axioms. If $h_p$ is a graphoid, then for any three disjoint sets $X,Y,Z \subseteq \Omega$, where $XYZ=\Omega$, the following holds:
	\begin{align*}
		h_p\models \Sigma_{pair} \implies (X;Y|Z)&& \text{ iff } && (X; Y|Z) \in \sigmapair^+ && \text{ iff } && X\bot_G Y|Z
	\end{align*}
\end{citedtheoremJAIR}
Theorem~\ref{thm:saturatedGraphoids} establishes that the graphoid axioms are sound and complete for inferring saturated CIs from $\sigmapair$ which, in turn, correspond to graph-separation in the independence graph. In Section~\ref{sec:intersectionAxiomDoesNotRelax}, we show that the intersection axiom does not relax. An immediate consequence is that no relaxation exists for exact implications involving the intersection axiom. In particular, no implication of the form $\sigmapair \implies (X;Y|Z)$ admits a relaxation. Practically, this means that the current method of inferring CIs in Markov Networks does not extend to the case where the CIs in $\Sigma_{pair}$ do not hold exactly.
\subsection{Bayesian Networks}
\label{sec:BNs}
Let $G(V,E)$ be a directed Acyclic Graph (DAG), and let $u,v \in V$. We say that $u$ is a \e{parent} of $v$, and $v$ a \e{child} of $u$ if $(u\rightarrow v) \in E$. A \e{directed path} $t=(v_1,v_2,\dots,v_n)$ is a sequence of vertices $(v_1,v_2,\dots,v_n)$ such that there is an edge $(v_i \rightarrow v_{i+1})\in E$ for every $i\in \set{1,\dots,n-1}$. We say that $v$ is a \e{descendant} of $u$, and $u$ an \e{ancestor} of $v$ if there is a directed path from $u$ to $v$.
A \e{trail} $t=(v_1,v_2,\dots,v_n)$ is a sequence of vertices $(v_1,v_2,\dots,v_n)$ such that there is an edge between $v_i$ and $v_{i+1}$ for every $i\in \set{1,\dots,n-1}$. That is, $(v_i\rightarrow v_{i+1})\in E$ or $(v_i \leftarrow v_{i+1})\in E$ for every $i\in \set{1,\dots,n-1}$. A vertex $v_i$ is said to be \e{head-to-head} with respect to $t$ if $(v_{i-1}\rightarrow v_i)\in E$ and $(v_{i}\leftarrow v_{i+1})\in E$.
A trail $t=(v_1,v_2,\dots,v_n)$ is \e{active} given $Z\subseteq V$ if (1) every $v_i$ that is a head-to-head vertex with respect to $t$ either belongs to $Z$ or has a descendant in $Z$, and (2) every $v_i$ that is not a head-to-head vertex with respect to $t$ does not belong to $Z$. If a trail $t$ is not active given $Z$, then it is \e{blocked} given $Z$. Let $X,Y,Z\subseteq V$ be pairwise disjoint. We say that $Z$ \e{$d$-separates} $X$ from $Y$ if every trail between $x\in X$ and $y\in Y$ is blocked given $Z$. We denote by $X\bot_{dsep} Y|Z$ that $Z$ $d$-separates $X$ from $Y$

A Bayesian network encodes the CIs of a probability distribution using a
DAG. Each node $X_i$ in a Bayesian network corresponds
to the variable $X_i\in \Omega$, a set of nodes $\alpha$ correspond to the set of variables $X_\alpha$, and $x_i \in \D_i$ is a
value from the domain of $X_i$. Each node $X_i$ in the network represents the distribution $p(X_i \mid X_{\pi(i)})$ where $X_{\pi(i)}$ is a set of variables that
correspond to the parent nodes $\pi(i)$ of $i$. The distribution represented by a Bayesian network is
\begin{equation}\label{eq:BN}
	p(x_1,\dots,x_n)=\prod_{i=1}^np(x_i| x_{\pi(i)})
\end{equation}
(when $i$ has no parents then $X_{\pi(i)}=\emptyset$).

Equation~\ref{eq:BN} implicitly encodes a set of $n$ conditional independence statements, called the \e{recursive basis} for the network:
\begin{equation}
	\label{eq:recursiveSet}
	\sigmarb\eqdef\set{(X_i; X_1\dots X_{i-1}{\setminus} \pi(X_i) \mid \pi(X_i)): i\in [n]}
\end{equation}
The implication problem associated with Bayesian Networks is to determine whether $\Gamma_n \models \sigmarb \implies \tau$ for a CI $\tau$. Let $G(\Omega,E)$ be a DAG generated by the recursive basis $\sigmarb$. That is, the vertices of $G$ are the RVs $\Omega$, and its edges are $E=\set{X_i\rightarrow X_j | X_i\in \pi(X_j)}$.
Given a CI $\tau=(A;B|C)$, the $d$-separation algorithm efficiently determines whether $C$ $d$-separates $A$ from $B$ in $G$. It has been shown that $\Gamma_n \models \sigmarb \implies \tau$ if and only if $C$ $d$-separates $A$ from $B$~\cite{DBLP:journals/networks/GeigerVP90}.
\begin{citedtheoremJAIR}{\cite{DBLP:journals/networks/GeigerVP90,DBLP:conf/uai/GeigerP88,DBLP:conf/uai/VermaP88}}
	\label{thm:soundnessCompletenessdSep}
	Let $f\in \Gamma_n$ be a polymatroid, and $G(\Omega,E)$ be the DAG generated by the recursive basis $\sigmarb$ (see~\eqref{eq:recursiveSet}). Let $\sigmarb^+$ denote the closure of $\sigmarb$ with respect to the semi-graphoid axioms. The following holds for every three disjoint sets $A,B,C \subseteq \Omega$:
	\begin{align*}
		\Gamma_n \models_{EI} \sigmarb \implies (A;B|C) && \text{ iff } && (A;B|C)\in \sigmarb^+ && \text{ iff } && A\bot_{dsep} B|C
	\end{align*} 
\end{citedtheoremJAIR}
Theorem~\ref{thm:soundnessCompletenessdSep} establishes that both the semi-graphoid axioms, and the $d$-separation criterion, are sound and complete for inferring CI statements from the recursive basis. Since the semigraphoid axioms follow from the Shannon inequalities (\eqref{eq:h:cond} and~\eqref{eq:h:mutual:information}), Theorem~\ref{thm:soundnessCompletenessdSep} estsablishes that the Shannon inequalities are both sound and complete for inferring CI statements from the recursive basis. In Section~\ref{sec:recursiveProof}, we show that the exact implication $\Gamma_n \models_{EI} \sigmarb \implies (A;B|C)$ admits a $1$-relaxation.

\section{Formal Statement of Results and Practical Implications}
\label{sec:results}
We begin by proving a negative result concerning the intersection axiom (see~\eqref{eq:intersection}).
\def\intersectionAxiomDoesNotRelax{
	For any finite $\lambda> 0$, there exists a strictly positive probability distribution $p_\lambda(A,B,C)$, such that:
	\begin{align*}
		\lambda\left(I_{h_{p_\lambda}}(A;B|C)+I_{h_{p_\lambda}}(A;C|B)\right) < I_{h_{p_\lambda}}(A;BC)
	\end{align*}
	where $h_{p_\lambda}$ is the entropy function of $p_\lambda$.
}
\begin{theorem}[The intersection axiom does not relax]
	\label{thm:intersection}
\intersectionAxiomDoesNotRelax
\end{theorem}
Theorem~\ref{thm:intersection} establishes that the intersection axiom (see~\eqref{eq:intersection}) does not relax, even when restricted to strictly positive distributions! This result has the following practical implication. Recall from theorem~\ref{thm:saturatedGraphoids} that every $XY$-separator $Z$ in the independence graph $G(V,E)$, generated using the saturated set of CIs $\sigmapair\eqdef \set{(u;v|V{\setminus}uv) :I_h(u;v|V{\setminus}uv)=0 }$, corresponds to the CI $I_h(X;Y|Z)=0$. Now, let $\varepsilon >0$, and suppose that we define $\sigmapair^\varepsilon \eqdef \set{(u;v|V{\setminus}uv) :I_h(u;v|V{\setminus}uv)\leq \varepsilon }$. We ask: for an $XY$-separator $S$ in $G$, can we bound the value $I_h(X;Y|S)$ ? More formally, is there a finite $0 < \lambda$, such that $I_h(X;Y|S)\leq \lambda{\cdot}h(\sigmapair^{\varepsilon})$ holds for all strictly positive probability distributions? Theorem~\ref{thm:intersection} establishes that the answer is negative. In other words, approximate CIs cannot be derived from $\sigmapair^{\varepsilon}$. 
In stark contrast to the negative result of Theorem~\ref{thm:intersection}, we show that approximate CIs can be derived using the semi-graphoid axioms, and the the $d$-separation algorithm in Bayesian networks.

We recall that a \e{conditional} is a statement of the form $A\rightarrow B$, which holds in a probability distribution $p$ if and only if $h(B|A)=0$.
\begin{theorem}
	\label{thm:saturated}
	Let $\Sigma$ be a set of saturated CIs and conditionals over the set of variables $\Omega \eqdef \set{X_1,\dots,X_n}$, and let $\tau\eqdef (A;B|C)$ be any CI. 
	\begin{align*}
		\Gamma_n \models_{EI} \Sigma \implies \tau && \text{ if and only if } && \Gamma_n \models I_h(A;B|C)\leq \min\set{|A|,|B|}h(\Sigma)
	\end{align*}
\end{theorem}
Theorem~\ref{thm:saturated} generalizes Theorem~\ref{thm:saturatedPolymatroids} in several aspects. 
The antecedents $\Sigma$ may contain both saturated CIs and conditionals, $\tau$ is not restricted to be a saturated CI, and most importantly, we prove that every exact implication from a set of saturated CIs and conditionals relaxes to the inequality $\Gamma_n \models I_h(A;B|C)\leq \min\set{|A|,|B|}h(\Sigma)$. Note that the only-if direction of Theorem~\ref{thm:saturated} is immediate, and follows from the non-negativity of Shannon's information measures. In previous work~\cite{DBLP:journals/lmcs/KenigS22}, it was shown that $\Gamma_n \models_{EI} \Sigma \implies \tau$ if and only if $\Gamma_n \models I_h(A;B|C)\leq |A|{\cdot}|B|{\cdot}h(\Sigma)\leq \frac{n^2}{4}h(\Sigma)$. Noting that $\min\set{|A|,|B|}\leq \frac{n}{2}$, the result of Theorem~\ref{thm:saturated} tightens the bound by an order of magnitude.
\begin{theorem}
	\label{thm:recursiveCIs}
	Let $\Sigma$ be a recursive set of CIs (see~\eqref{eq:recursiveSet}), and let $\tau = (A;B|C)$. Then the following holds: 
	\begin{align}
		\Gamma_n \models_{EI} \Sigma \Rightarrow \tau& &\text{ if and only if } & &\Gamma_n \models h(\Sigma) \geq h(\tau)
	\end{align}
\end{theorem}
The result of Theorem~\ref{thm:recursiveCIs} has the following practical implication. Theorem~\ref{thm:soundnessCompletenessdSep} establishes that if $X$ and $Y$ are $d$-separated given $Z$ in the DAG $G(V,E)$, generated by the recursive basis $\sigmarb$ (see~\eqref{eq:recursiveSet}), then $I_h(X;Y|Z)=0$. Now, let $\varepsilon >0$, and suppose that for every $(X_i;X_1\dots X_{i-1}|\pi(X_i))\in \sigmarb$, it holds that $I_h(X_i;X_1\dots X_{i-1}|\pi(X_i))\leq \varepsilon$. We ask: can we bound the value of $I_h(X;Y|Z)$? Theorem~\ref{thm:recursiveCIs} establishes that $I_h(X;Y|Z)\leq h(\sigmarb)$. Theorem~\ref{thm:recursiveCIs} was first proved in~\cite{DBLP:conf/uai/Kenig21}. In this paper, we simplify the proof, and relate it to the $d$-separation algorithm.

Finally, we consider implications from the set of marginal CIs.
\begin{thm}
	\label{thm:marginal}
	Let $\Sigma$ be a set of marginal CIs, and $\tau=(A;B|C)$ be any CI.
	\begin{align}
		\Gamma_n \models_{EI} \Sigma \Rightarrow \tau& &\text{ if and only if }& &\Gamma_n \models   I_h(A;B|C)\leq \min\set{|A|,|B|}h(\Sigma)
	\end{align}
\end{thm}
Theorem~\ref{thm:marginal} generalizes the result of~\cite{GEIGER1991128}, which proved that the semi-graphoid axioms are sound and complete for deriving marginal CIs. In previous work~\cite{DBLP:conf/uai/Kenig21}, it was shown that $\Gamma_n \models_{EI} \Sigma \implies \tau$ if and only if $\Gamma_n \models I_h(A;B|C)\leq |A|{\cdot}|B|{\cdot}h(\Sigma)\leq \frac{n^2}{4}h(\Sigma)$. Noting that $\min\set{|A|,|B|}\leq \frac{n}{2}$, the result of Theorem~\ref{thm:marginal} tightens the bound by an order of magnitude.

\def\dom{\mathcal{D}}
\section{Intersection Axiom Does not Relax}
\label{sec:intersectionAxiomDoesNotRelax}
It is well-known that the intersection axiom~\eqref{eq:intersection} does not hold for all probability distributions. From this, we can immediately conclude that the intersection axiom does not relax for all polymatroids. This follows from the fact that the entropic function associated with any probability distribution is a polymatroid, and that approximate implication generalizes exact implication. In this section, we prove Theorem~\ref{thm:intersection}, establishing that the intersection axiom does not relax even for strictly positive probability distributions. In other words, we prove the (surprising) result that even for the class of distributions in which the intersection axiom holds, it does not relax.

Let $p(A,B,C)$ be a strictly positive probability distribution, where $I_h(A;B|C)=0$ and $I_h(A;C|B)=0$. According to the intersection axiom~\eqref{eq:intersection}, it holds that $I_h(A;BC)=0$. \eat{Graphically, this is depicted in Figure~\ref{fig:simpleCE}.
	\begin{figure}
		\centering
		\includegraphics[width=0.25\textwidth]{ABCGraph.pdf}
		\caption{Let $P(A,B,C)$ be a strictly positive joint probability distribution such that $I(A;B|C)=I(A;C|B)=0$, and $I(B;C|A)>0$. Hence, $P$ is (exactly) pairwise Markov with respect to $G$.
			Since $\emptyset$ separates $A$ from $B$ in $G$, then by Theorem~\ref{thm:HC}, we have that $I(A;B|\emptyset)=0$.}
		\label{fig:simpleCE}
	\end{figure}
}

To prove the Theorem, we describe the following ``helper'' distribution. Let $A_1,\dots,A_7$ denote mutually independent, random variables. The random variable $A_7$ is defined $A_7\eqdef(A_{7,1},A_{7,2},A_{7,3})$, where $A_{7,j}$ is a binary RV for all $j\in \set{1,2,3}$. For $i\in \set{4,5,6}$, $A_i\eqdef(A_{i,1},A_{i,2})$, where $A_{i,1}$ and $A_{i,2}$ are binary RVs. The joint distributions of $A_7\eqdef (A_{7,1},A_{7,2},A_{7,3})$, and $A_i\eqdef (A_{i,1},A_{i,2})$ for $i\in \set{4,5,6}$, are defined as follows:
	\begin{table}[H]	
		\parbox{0.45\linewidth}{
			\centering
			\begin{tabular}{c c c | c}	
				$A_{7,1}$ & $A_{7,2}$ & $A_{7,3}$ & $P$ \\
				\hline
				$0$ & $0$ & $0$ & $\frac{1}{2}-3y$ \\
				$0$ & $0$ & $1$ & $y$ \\
				$0$ & $1$ & $0$ & $y$ \\
				$0$ & $1$ & $1$ & $y$ \\
				$1$ & $0$ & $0$ & $y$ \\
				$1$ & $0$ & $1$ & $y$ \\
				$1$ & $1$ & $0$ & $y$ \\
				$1$ & $1$ & $1$ & $\frac{1}{2}-3y$ 
			\end{tabular}
			\caption{Joint distribution of $A_7\eqdef (A_{7,1},A_{7,2},A_{7,3})$, where $y\in (0,\frac{1}{6})$.}
			\label{tab:A7}
		}
		\hfill
		\parbox{0.45\linewidth}{
			\centering
			\begin{tabular}{c c | c}
				$A_{i,1}$ & $A_{i,2}$ & $P$ \\
				\hline
				$0$ & $0$ & $1-3x$ \\
				$0$ & $1$ & $x$ \\
				$1$ & $0$ & $x$ \\
				$1$ & $1$ & $x$
			\end{tabular}
			\caption{Joint distribution of $A_i\eqdef(A_{i,1},A_{i,2})$ for $i\in \set{4,5,6}$, where $x\in (0,\frac{1}{3})$.}
			\label{tab:A456}
		}
	\end{table}
	Finally, for $A_1,A_2,A_3$ we have that $P(A_i=1)=P(A_i=0)=\frac{1}{2}$.

	We define the RVs $A,B$, and $C$ as follows:
	\begin{align}
		A & \eqdef (A_2,A_{6,1},A_{7,1}, A_{5,1}) \label{eq:A}\\
		B & \eqdef (A_3,A_{6,2},A_{7,2}, A_{4,1}) \label{eq:B}\\
		C & \eqdef (A_1,A_{5,2},A_{7,3}, A_{4,2}) \label{eq:C}
	\end{align}
	
	The information diagram corresponding to $p(A,B,C)$ is presented in Figure~\ref{fig:InformationDiagram}.
	
	\begin{figure}
		\centering
		\includegraphics[width=0.3\textwidth]{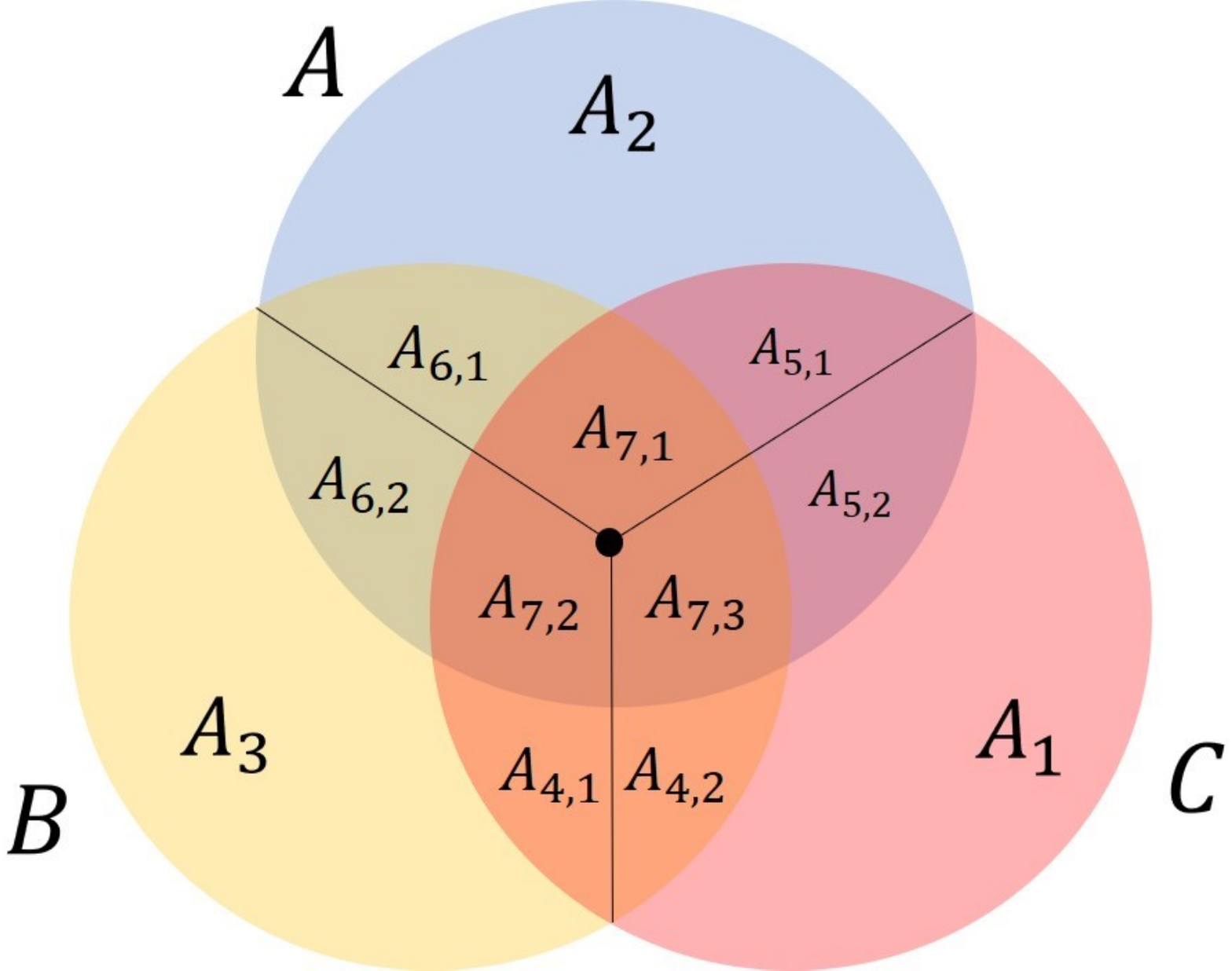}
		\caption{The information diagram for the joint probability $p(A,B,C)$ where $A,B$ and $C$ are defined in~\eqref{eq:A}-\eqref{eq:C}.}
		\label{fig:InformationDiagram}
	\end{figure}
	
	\begin{lemma}
		\label{lem:positive}
		The joint distribution $P: \dom(A)\times \dom(B)\times \dom(C) \rightarrow (0,1)$ is strictly positive.
	\end{lemma}
	\begin{proof}
		Take any assignment $d \in\dom(A)\times \dom(B)\times \dom(C)$. Since $A$, $B$, and $C$ have no common RVs, then $d$ uniquely maps to an assignment to the binary RVs $A_1$,$A_2$,$A_3$,$A_{4,1}$,$A_{4,2}$,$A_{5,1}$, etc.\eat{$A_{5,2}$,$A_{6,1}$,$A_{6,2}$,$A_{7,1}$,$A_{7,2}$, and $A_{7,3}$.} In particular, $d$ maps to a unique assignment to $A_1,A_2,A_3,A_{4},A_5,A_6$, and $A_7$. For all $i\in \set{1,\dots,7}$, denote by $d_i$ the assignment to $A_i$ induced by $d$. 
		By definition, these RVs are mutually independent. Hence:
		\begin{align*}
			P(d)&=P(A_1=d_1,A_2=d_2,A_3=d_3,A_4=d_4,A_5=d_5,A_6=d_6,A_7=d_7)\\
			&=P(A_1=d_1) \cdots  P(A_7=d_7)
		\end{align*}
		By definition, for every $A_i$ it holds that $p(A_i)$ is strictly positive (see Tables~\ref{tab:A7} and~\ref{tab:A456}). This proves the claim.
	\end{proof}
	
	We define the following functions where $x\in (0,\frac{1}{3})$ and $y\in (0,\frac{1}{6})$:
	\begin{align}
		\delta_1(x)&\eqdef -\left((1-3x)\log(1-3x)+3x\log x\right)&& \text{ where }x\in (0,\frac{1}{3}) \label{eq:delta1}\\
		\delta_2(x)& \eqdef -\left((1-2x)\log(1-2x)+2x\log 2x\right)&& \text{ where }x\in (0,\frac{1}{3}) \label{eq:delta2}\\
		f_1(y)& \eqdef -\left(2(\frac{1}{2}-3y)\log(\frac{1}{2}-3y)+6y\log y \right)&& \text{ where }y\in (0,\frac{1}{6}) \label{eq:f1}\\
		f_2(y)& \eqdef -\left(2(\frac{1}{2}-2y)\log(\frac{1}{2}-2y)+4y\log 2y \right)&& \text{ where }y\in (0,\frac{1}{6}) \label{eq:f2} 
	\end{align}
	
	\begin{lemma}
		\label{lem:limits}
		It holds that:
		\begin{align}
			& \lim_{x\rightarrow 0}\delta_1(x)=\lim_{x\rightarrow 0}\delta_2(x)=0\\
			& \lim_{y\rightarrow 0}f_1(y)=\lim_{y\rightarrow 0}f_2(x)=1
		\end{align}
	\end{lemma}
	\begin{proof}
		Both $x$ and $\log x$ are differentiable on $(0,\infty)$. Therefore, we can apply LHospital's rule to get that $\lim_{x\rightarrow 0}x\log x=0$, and that $\lim_{y\rightarrow 0}y\log y=0$. Therefore:
		\begin{align*}
			\lim_{x\rightarrow 0}\delta_1(x)&=-\left(\lim_{x\rightarrow 0}((1-3x)\log(1-3x)+\lim_{x\rightarrow 0}3x\log x)\right)=-\left(1\log 1+0\right)=0\\
			\lim_{x\rightarrow 0}\delta_2(x)&=-\left(\lim_{x\rightarrow 0}((1-2x)\log(1-2x)+\lim_{x\rightarrow 0}2x\log 2x)\right)=-\left(1\log 1+0\right)=0\\
			\lim_{y\rightarrow 0}f_1(y)&=-\left(\lim_{y\rightarrow 0}(2(\frac{1}{2}-3y)\log(\frac{1}{2}-3y)+\lim_{y\rightarrow 0}6y\log y)\right)=-\left(\log\frac{1}{2}+0\right)=1\\
			\lim_{y\rightarrow 0}f_2(y)&=-\left(\lim_{y\rightarrow 0}(2(\frac{1}{2}-2y)\log(\frac{1}{2}-2y)+\lim_{y\rightarrow 0}4y\log 2y)\right)=-\left(\log\frac{1}{2}+0\right)=1
		\end{align*}
	\end{proof}
\def\technicalLemmaEntropy{
			The following holds for any $x\in (0,\frac{1}{3})$, and $y\in (0,\frac{1}{6})$.  
	\begin{enumerate}
		\setlength{\itemsep}{0.1em}
		\item $H(A_i)=1$ for $i\in \set{1,2,3}$.
		\item $H(A_i)=H(A_{i,1},A_{i,2})=\delta_1(x)$ for $i\in \set{4,5,6}$.
		\item $H(A_{i,1})=H(A_{i,2})=\delta_2(x)$ for $i\in \set{4,5,6}$.	
		\item $H(A_{7,j})=1$ for $j\in \set{1,2,3}$.
		\item $H(A_{7,j},A_{7,k})=f_2(y)$ for $j\neq k$ and $j,k\in \set{1,2,3}$.
		\item $H(A_{7})=H(A_{7,1},A_{7,2},A_{7,3})=f_1(y)$.
	\end{enumerate}
	where $\delta_1,\delta_2,f_1$, and $f_2$ are defined in~\eqref{eq:delta1}-\eqref{eq:f2}.
}
	\begin{lemma}
	%	\label{lem:technical1}
		\label{lem:technicalLemmaEntropy}
		\technicalLemmaEntropy
	\end{lemma}
\begin{proofOverview}
	The proof follows from the definition of the RVs $A_1,A_2,\dots,A_7$ in Tables~\ref{tab:A7}, and~\ref{tab:A456}. The technical details are deferred to Section~\ref{sec:IntersectionTechnicalDetailes} in the Appendix.
\end{proofOverview}
\eat{
	\begin{proof}
		The proof is by definition, and we provide the technical details.
		For $i\in \set{1,2,3}$:
		\[
		H(A_i)=2\cdot\frac{1}{2}\log 2= 1
		\]
		By the definition of $A_4$ in Table~\ref{tab:A456}, we have that:
		\[
		H(A_4)=-\left((1-3x)\log(1-3x)+3x\log x\right)=\delta_1(x).
		\]
		Proof is the same for $H(A_5)$ and $H(A_6)$.
		
		We now compute $H(A_{4,1})$. From Table~\ref{tab:A456}, we have that $P(A_{4,1}=0)=1-3x+x=1-2x$, and $P(A_{4,1}=1)=2x$. Therefore,
		\[
		H(A_{4,1})=-\left((1-2x)\log(1-2x)+2x\log 2x\right)=\delta_2(x).
		\]
		For symmetry reasons, the same holds for $H(A_{4,2})$ and $H(A_{i,j})$ for $i\in \set{5,6}$ and $j\in \set{1,2}$.
		
		From Table~\ref{tab:A7}, we have that for all $i\in \set{1,2,3}$:
		\[
		P(A_{7,i}=0)=P(A_{7,i}=1)=\frac{1}{2}
		\]
		Therefore, $H(A_{7,i})=1$ for all $i\in \set{1,2,3}$.
		
		Now, take any $i,j\in \set{1,2,3}$ where $i<j$. Then, from Table~\ref{tab:A7}, we have that:
		\begin{equation*}
			P(A_{7,i}=a,A_{7,j}=b)=
			\begin{cases}
				\frac{1}{2}-2y & \text{if }a=b \\
				2y & \text{ otherwise }
			\end{cases}
		\end{equation*}
		Therefore, we have that:
		\[
		H(A_{7,i},A_{7,j})=-\left(2(\frac{1}{2}-2y)\log(\frac{1}{2}-2y)+2\cdot 2y\log 2y \right)=f_2(y)
		\]
		Finally, by Table~\ref{tab:A7}, we have that:
		\[
		H(A_7)=H(A_{7,1},A_{7,2},A_{7,3})=-\left(2(\frac{1}{2}-3y)\log(\frac{1}{2}-3y)+ 6y\log y\right)=f_1(y)
		\]
	\end{proof}
}

	\def\mainTechnicalLemmaIntersection{
			The following holds:
		\begin{enumerate}
			\setlength{\itemsep}{0.1em}
			\item $H(A)=H(B)=H(C)=2+2\delta_2(x)$
			\item $H(AC)=H(AB)=H(BC)=2+ \delta_1(x)+2\delta_2(x)+f_2(y)$
			\item $H(ABC)=3+3\delta_1(x)+f_1(y)$
		\end{enumerate}
	}
	\begin{lemma}
		\label{lem:mainTechnicalLemmaIntersection}
		\mainTechnicalLemmaIntersection
	\end{lemma}
\begin{proofOverview}
	The proof follows from the definition of the RVs $A_1,A_2,\dots,A_7$ (see tables~\ref{tab:A7}, and~\ref{tab:A456}), the fact that they are mutually independent, the definition of RVs $A$, $B$, and $C$ (see~\eqref{eq:A}-\eqref{eq:C}), and the application of Lemma~\ref{lem:technicalLemmaEntropy}.  The technical details are deferred to Section~\ref{sec:IntersectionTechnicalDetailes} in the Appendix.
\end{proofOverview}
	An immediate consequence from Lemma~\ref{lem:mainTechnicalLemmaIntersection} is that:
	\begin{align}
		I(A;B|C)&=H(AC)+H(BC)-H(C)-H(ABC)\\
		&=2(\delta_1(x)+f_2(y)+2+2\delta_2(x))-(2+2\delta_2(x))-(3+3\delta_1(x)+f_1(y))\nonumber \\
		&=2\delta_2(x)-\delta_1(x)+2f_2(y)-f_1(y)-1 \label{eq:IABgC}
	\end{align}
	By symmetry, $I(A;C|B)=I(B;C|A)=2\delta_2(x)-\delta_1(x)+2f_2(y)-f_1(y)-1$ as well.
	And, that:
	\begin{align}
		I(A;B)&=H(A)+H(B)-H(AB) \nonumber \\
		&=2(2+2\delta_2(x))-(2+\delta_1(x)+2\delta_2(x)+f_2(y)) \nonumber \\
		&=4+4\delta_2(x)-2-\delta_1(x)-2\delta_2(x)-f_2(y) \nonumber \\	
		&=2\delta_2(x)-\delta_1(x)-f_2(y)+2 \label{eq:IAB}
	\end{align}
	and, by symmetry, $I(A;C)=I(B;C)=2\delta_2(x)-\delta_1(x)-f_2(y)+2$ as well.

	\begin{reptheorem}{\ref{thm:intersection}}
		\intersectionAxiomDoesNotRelax
	\end{reptheorem}
\eat{
	\begin{theorem}[The Intersection Axiom does not relax]
		For every finite $\lambda>0$, there exists a 
		strictly positive joint probability distribution $p:\dom(A)\times \dom(B)\times \dom(C) \rightarrow (0,1)$ such that 
		\[
		\lambda(I(A;B|C)+I(A;C|B)) < I(A;B)
		\]
	\end{theorem}
}
	\begin{proof}
		Suppose otherwise, and consider the joint probability distribution $p:\dom(A)\times \dom(B)\times \dom(C) \rightarrow (0,1)$, where $A$, $B$ and $C$ are defined in~\eqref{eq:A}-\eqref{eq:C}. By Lemma~\ref{lem:positive}, $p$ is a  strictly positive distribution for all $x\in (0,\frac{1}{3})$, and $y\in (0,\frac{1}{6})$.
		Then there exists a fixed, finite $\lambda$ such that for all $x\in (0,\frac{1}{3})$, and all $y\in (0,\frac{1}{6})$ it holds that:
		\[
		\lambda(I(A;B|C)+I(A;C|B))\geq I(A;B)
		\]
		From~\eqref{eq:IABgC} and~\eqref{eq:IAB}, it means that for all $x\in (0,\frac{1}{3})$, and all $y\in (0,\frac{1}{6})$ it holds that:
		\begin{equation}
			\label{eq:thmProof1}
			2\lambda(2\delta_2(x)-\delta_1(x)+2f_2(y)-f_1(y)-1)\geq 2\delta_2(x)-\delta_1(x)-f_2(y)+2
		\end{equation}
		By the assumption,~\eqref{eq:thmProof1} holds for all $x\in (0,\frac{1}{3})$, and all $y\in (0,\frac{1}{6})$. In particular, it holds in the limits $x\rightarrow 0$ and $y\rightarrow 0$. That is:
		\begin{align}	
			\lim_{x\rightarrow 0, y\rightarrow 0}\left(2\lambda(2\delta_2(x)-\delta_1(x)+2f_2(y)-f_1(y)-1)\right)&\geq \lim_{x\rightarrow 0, y\rightarrow 0}\left(2\delta_2(x)-\delta_1(x)-f_2(y)+2\right)\nonumber \\
			2\lambda\left(2\cdot 0-0+2-1-1\right)&\geq  \left(2\cdot 0-0-1+2\right) \label{eq:thmProof3}\\
			2\lambda \cdot 0 & \geq 1 \nonumber
		\end{align}
		The transition in~\eqref{eq:thmProof3} is due to Lemma~\ref{lem:limits}.
		This brings us to a contradiction. Hence, no such finite $\lambda$ exists, and the intersection axiom does not relax.
	\end{proof}

\section{Properties of Exact Implication}
\label{sec:lemmas}
This section proves various technical lemmas that establish some general properties of exact implication in the set $\Gamma_n$ of $n$-dimensional polymatroids, and a certain subset of polymatroids called \e{positive polymatroids}, to be defined later.
The lemmas in this section will be used for proving the approximate implication guarantees presented in Section~\ref{sec:results}.
A central tool in our analysis of exact and approximate implication is the \e{I-measure theory}~\cite{DBLP:journals/tit/Yeung91,Yeung:2008:ITN:1457455}. We present the required background on the I-Measure theory in Section~\ref{sec:imeasure}. 
\eat{
In this section, we use the I-measure to characterize some general properties of exact implication in the set of positive polymatroids $\Delta_n$ (Section~\ref{sec:implicationPositivePolymatroids}), and the entire set of polymatroids $\Gamma_n$ (Section~\ref{sec:EIPolymatroids}). The lemmas in this section will be used for proving the approximate implication guarantees presented in Section~\ref{sec:results}.
}

In what follows, $\Omega\eqdef\set{X_1,\dots,X_n}$ is a set of $n$ jointly-dstributed RVs, $\Sigma$ denotes a set of triples $(A;B|C)$, and $\tau$ denotes a single triple. We denote by $\vars(\sigma)$ the set of RVs mentioned in $\sigma$ (e.g., if $\sigma=(X_1X_2;X_3|X_4)$ then $\vars(\sigma)=\set{X_1,X_2,X_3,X_4}$).

\subsection{The I-measure}
\label{sec:imeasure}
The I-measure~\cite{DBLP:journals/tit/Yeung91,Yeung:2008:ITN:1457455} is a theory which establishes a one-to-one correspondence between Shannon's information measures and set theory. 
Let $h\in \entropicPlhdrl_n$ denote a polymatroid defined over the variables $\set{X_1,\dots,X_n}$. Every variable $X_i$ is associated with a set $\iset(X_i)$, and it's complement $\isetc(X_i)$.
The universal set is $\Lambda \eqdef \bigcup_{i=1}^n\iset(X_i)$.
Let $\alpha \subseteq [n]$. We denote by $X_\alpha\eqdef\set{X_j \mid j \in \alpha}$. For the variable-set $X_\alpha$, we define:
\begin{align}
	\iset(X_\alpha)\eqdef \bigcup_{i\in \alpha}\iset(X_i) &&  \isetc(X_\alpha)\eqdef \bigcap_{i\in \alpha}\isetc(X_i)
\end{align}

\begin{citeddefnJAIR}{\cite{DBLP:journals/tit/Yeung91,Yeung:2008:ITN:1457455}} \label{thm:YeungUniqueness}\label{def:field}
	The field $\mathcal{F}_n$ generated by sets $\iset(X_1),\dots,\iset(X_n)$ is the collection of sets which can be obtained by any sequence of usual set operations (union, intersection, complement, and difference) on $\iset(X_1),\dots,\iset(X_n)$.
\end{citeddefnJAIR}

The \e{atoms} of $\mathcal{F}_n$ are sets of the form $\bigcap_{i=1}^nY_i$, where $Y_i$ is either $\iset(X_i)$ or $\isetc(X_i)$. We denote by $\mathcal{A}$ the atoms of $\mathcal{F}_n$.
We consider only atoms in which at least one set appears in positive form (i.e., the atom $\bigcap_{i=1}^n\isetc(X_i)\eqdef \emptyset$ is empty).
There are $2^n-1$ non-empty atoms and $2^{2^n-1}$ sets in $\mathcal{F}_n$ expressed as the union of its atoms. 
A function $\measure:\mathcal{F}_n \rightarrow \real$ is \e{set additive} if, for every pair of disjoint sets $A$ and $B$, it holds that $\measure(A\cup B)=\measure(A)\vplus\measure(B)$.
A real function $\measure$ defined on $\mathcal{F}_n$ is called a \e{signed measure} if it is set additive, and $\measure(\emptyset)=0$.

The $I$-measure $\imeasure$ on $\mathcal{F}_n$ is defined by $\imeasure(m(X_\alpha))\eqdef H(X_\alpha)$ for all nonempty subsets $\alpha \subseteq \set{1,\dots,n}$, where $H$ is the entropy~\eqref{eq:entropy}. 
Table~\ref{tab:ImeasureSummary} summarizes the extension of this definition to the rest of the Shannon measures.
\begin{table}[]
	\centering
	\small
	\begin{tabular}{|c|c|}	
		\hline
		Information & \multirow{2}{*}{$\imeasure$} \\
		Measures &  \\ \hline
		$H(X)$	& $\imeasure(\iset(X))$ \\ \hline
		$H(XY)$	& $\imeasure\left(\iset(X)\cup\iset(Y)\right)$ \\ \hline
		$H(X|Y)$	& $\imeasure\left(\iset(X)\cap \isetc(Y)\right)$ \\ \hline
		$I_H(X;Y)$ & $\imeasure\left(\iset(X)\cap \iset(Y)\right)$  \\ \hline
		$I_H(X;Y|Z)$	&  $\imeasure\left(\iset(X)\cap \iset(Y) \cap \isetc(Z)\right)$ \\ \hline
	\end{tabular}
	\vspace{0.2cm}
	\caption{Information measures and associated I-measure}
	\label{tab:ImeasureSummary}
\end{table}
Yeung's I-measure Theorem establishes the one-to-one correspondence between Shannon's information measures and $\imeasure$.
\begin{citedtheoremJAIR}{\cite{DBLP:journals/tit/Yeung91,Yeung:2008:ITN:1457455}} \label{thm:YeungUniqueness}\e{[I-Measure Theorem]}
	$\imeasure$ is the unique signed measure on $\mathcal{F}_n$ which is consistent with all Shannon's information measures (i.e., entropies, conditional entropies, and mutual information). 	
\end{citedtheoremJAIR}

Let $\sigma=(X;Y|Z)$. We denote by $\iset(\sigma)\eqdef\iset(X)\cap\iset(Y)\cap\isetc(Z)$ the set associated with $\sigma$ (see Table~\ref{tab:ImeasureSummary}). For a set of triples $\Sigma$, we define:
\begin{equation}
	\label{eq:SigmaSet}
	\iset(\Sigma)\eqdef\bigcup_{\sigma \in \Sigma}\iset(\sigma)
\end{equation}
\begin{example}
	Let $A$, $B$, and $C$ be three disjoint sets of RVs defined as follows: $A{=}A_1A_2A_3$, $B{=}B_1B_2$ and $C{=}C_1C_2$. Then, by Theorem~\ref{thm:YeungUniqueness}: $H(A){=}\mu^*(\iset(A)){=}\mu^*(\iset(A_1){\cup}\iset(A_2){\cup}\iset(A_3))$, $H(B){=}\mu^*(\iset(B)){=}\mu^*(\iset(B_1){\cup}\iset(B_2))$, and $\mu^*(\isetc(C)){=}\mu^*(\isetc(C_1){\cap} \isetc(C_2))$. By Table~\ref{tab:ImeasureSummary}, $I(A;B|C){=}\mu^*(\iset(A)\cap \iset(B)\cap \isetc(C))$.
\end{example}
We denote by $\Delta_n$ the set of signed measures $\mu^*:\mathcal{F}_n \rightarrow \real_{\geq 0}$ that assign non-negative values to the atoms $\mathcal{F}_n$. We call these \e{positive I-measures}.
\begin{citedtheoremJAIR}{\cite{Yeung:2008:ITN:1457455}}\label{thm:YeungBuildMeasure}
	If there is no constraint on $X_1,\dots,X_n$, then $\imeasure$ can take any set of nonnegative values on the nonempty atoms of $\mathcal{F}_n$.
\end{citedtheoremJAIR}	
Theorem~\ref{thm:YeungBuildMeasure} implies that every positive I-measure $\mu^*$ corresponds to a function that is consistent with the Shannon inequalities, and is thus a polymatroid. Hence, $\Delta_n\subset \Gamma_n$ is the set of polymatroids with a positive I-measure that we call \e{positive polymatroids}.

\subsection{Exact implication in the set of positive polymatroids}
\label{sec:implicationPositivePolymatroids}
\def\implicationInclusion{
	\eat{	Let $\Sigma$ denote a set of mutual information terms, and $\tau$ a single mutual information term. Then:
	}
	The following holds:
	\begin{align*}
		\eat{	 \text{if  } && \Gamma_n \models_{EI} \Sigma \implies \tau && \text{ then }&& \iset(\Sigma) \supseteq \iset(\tau) \\}
		&&\Delta_n \models_{EI} \Sigma \implies \tau&& \text{ if and only if }&& \iset(\Sigma) \supseteq \iset(\tau) 		
	\end{align*}
}
\begin{lemma}\label{lem:implicationInclusion}
	\implicationInclusion
\end{lemma}
\begin{proof}
 Suppose that $\iset(\impliedCI) \not\subseteq \iset(\Sigma)$, and let $b\in \mathcal{F}_n$ be an atom such that $b\in \iset(\impliedCI){\setminus}\iset(\Sigma)$. By Theorem~\ref{thm:YeungBuildMeasure}, there exists a positive polymatroid in $\Delta_n$ with an $I$-measure $\imeasure$ that takes the following non-negative values on its atoms: $\imeasure(b)=1$, and $\imeasure(a)=0$ for any atom $a\in \mathcal{F}_n$ where $a\neq b$.
	Since $b\notin \iset(\Sigma)$, then $\imeasure(\Sigma)=0$ while $\imeasure(\impliedCI)= 1$. Hence, $\Delta_n \not\models\Sigma \implies \tau$.\eat{ which contradicts the exact implication.}
	
	Now, suppose that $\iset(\Sigma) \supseteq \iset(\tau)$.
	Then for any positive $I$-measure $\imeasure: \mathcal{F}_n {\rightarrow} \real_{\geq 0}$, we have that $\imeasure(\iset(\Sigma)) \geq \imeasure(\iset(\tau))$. By Theorem~\ref{thm:YeungUniqueness}, $\mu^*$ is the unique signed measure on $\mathcal{F}_n$ that is consistent with all of Shannon's information measures. Therefore, $h(\Sigma) \geq h(\tau)$. The result follows from the non-negativity of the Shannon information measures.
\end{proof}
An immediate consequence of Lemma~\ref{lem:implicationInclusion} is that $\iset(\Sigma) \supseteq \iset(\tau)$ is a necessary condition for implication between polymatroids. 
\begin{corr}
	\label{corr:inclusionGamman}
	If $\Gamma_n \models_{EI} \Sigma \implies \tau$ then $\iset(\Sigma) \supseteq \iset(\tau)$.
\end{corr}
\begin{proof}
	If $\Gamma_n \models_{EI} \Sigma \implies \tau$ then it must hold for any subset of polymatroids, and in particular, $\Delta_n \models_{EI} \Sigma \implies \tau$. The result follows from Lemma~\ref{lem:implicationInclusion}.
\end{proof}
\eat{
	In what follows, we recall the definition of $\iset(\sigma)$ from Table~\ref{tab:ImeasureSummary}.
}
\begin{lemma}
	\label{lem:excludeSigma}
	Let $\Delta_n \models_{EI} \Sigma \implies \tau$, and let $\sigma \in \Sigma$ such that $\iset(\sigma)\cap \iset(\tau)=\emptyset$. Then $\Delta_n \models_{EI} \Sigma{\setminus}\set{\sigma} \implies \tau$.
\end{lemma}
\begin{proof}
	Let $\Sigma'=\Sigma {\setminus} \set{\sigma}$, and suppose that $\Delta_n \not\models_{EI} \Sigma' \implies \tau$. By Lemma~\ref{lem:implicationInclusion}, we have that $\iset(\Sigma')\not\supseteq \iset(\tau)$. In other words, there is an atom $a\in \mathcal{F}_n$ such that $a\in \iset(\tau){\setminus} \iset(\Sigma')$.
	In particular, $a \notin \iset(\sigma) \cup \iset(\Sigma')= \iset(\Sigma)$. Hence, $\iset(\tau) \not\subseteq \iset(\Sigma)$, and by Lemma~\ref{lem:implicationInclusion} we get that $\Delta_n \not\models_{EI} \Sigma \implies \tau$, which brings us to a contradiction.
\end{proof}

\begin{corr}
	\label{corr:excludeSigma}
	Let $\Delta_n \models_{EI} \Sigma \implies \tau$ where $\tau=(A;B|C)$, and let $\sigma=(X;Y|Z)\in \Sigma$. If $A\subseteq Z$, $B\subseteq Z$, $X\subseteq C$, or $Y\subseteq C$, then $\Delta_n \models_{EI} \Sigma{\setminus}\set{\sigma} \implies \tau$.
\end{corr}
\begin{proof}
By definition, it holds that $\iset(\tau)=\iset(A)\cap \iset(B)\cap \isetc(C)=(\cup_{a\in A}\iset(a))\cap (\cup_{b\in B}\iset(b))\cap(\cap_{c\in C}\isetc(c))$, and likewise $\iset(\sigma)=(\cup_{x\in X}\iset(x))\cap (\cup_{y\in Y}\iset(y))\cap(\cap_{z\in Z}\isetc(z))$. If $A\subseteq Z$, then $\iset(\sigma)\subseteq \isetc(Z) \subseteq \isetc(A)$, while $\iset(\tau)\subseteq \iset(A)$. Therefore, $\iset(\tau)\cap \iset(\sigma)=\emptyset$. Similarly, if $B\subseteq Z$, then $\iset(\sigma)\subseteq \isetc(Z) \subseteq \isetc(B)$, while $\iset(\tau)\subseteq \iset(B)$, and hence $\iset(\tau)\cap \iset(\sigma)=\emptyset$.  Similarly, it is shown that if $X\subseteq C$ or $Y\subseteq C$, then $\iset(\tau)\cap \iset(\sigma)=\emptyset$.  The corollary then directly follows from Lemma~\ref{lem:excludeSigma}.
\end{proof}

\begin{lemma}
	\label{lem:simpleExistenceLemma}
	Let $\Delta_n \models_{EI} \Sigma \implies \tau=(A;B|C)$. There exists a CI $\sigma=(X;Y|Z)\in \Sigma$ such that $X\not\subseteq C$, $Y\not\subseteq C$, and $Z\subseteq C$. 
\end{lemma}
\begin{proof}
Let $\Sigma'\eqdef \set{(X;Y|Z) \in \Sigma \mid X\not\subseteq C \text{ and }Y\not\subseteq C}$. By Corollary~\ref{corr:excludeSigma}, if $\Delta_n \models_{EI} \Sigma \implies \tau$, then $\Delta_n \models_{EI} \Sigma' \implies \tau$. Suppose, by way of contradiction, that for every $\sigma=(X;Y|Z)\in \Sigma'$, it holds that $Z\not\subseteq C$. Consider the atom $t\eqdef \bigcap_{a\in \Omega{\setminus}C}\iset(a)\cap \isetc(C)$. Clearly, $t\in \iset(\tau)$. Now, take any $\sigma=(X;Y|Z)\in \Sigma'$. Since $Z\not\subseteq C$, then there exists a variable $a\in Z{\setminus}C$, and $\iset(\sigma)\subseteq \isetc(Z)\subseteq \isetc(a)$. On the other hand, since $a\notin C$, then $t\in \iset(a)$, and hence $t\notin \iset(\sigma)$. Since this is the case for all $\sigma \in \Sigma'$, then $t\notin \iset(\Sigma')$, and by Lemma~\ref{lem:implicationInclusion}, that $\Delta_n \not\models \Sigma'\implies \tau$, and this brings us to a contradiction. 
\end{proof}
\subsection{Exact Implication in the set of polymatroids}
\label{sec:EIPolymatroids}
The main technical result of this section is Lemma~\ref{lem:implicationCovers}. We start with a short technical lemma.

\def\chainRuleTechnicalLemma{ Let $\sigma = (X;Y|Z)$ and
	$\tau = (A;B|C)$ be CIs such that $A\subseteq X$, $B\subseteq Y$,
	$Z \subseteq C$, and $C \subseteq XYZ$. Then,
	$\entropicPlhdrl_n \models h(\tau) \leq h(\sigma)$.  }

\begin{lemma}\label{lem:chainRuleTechnicalLemma}
	\chainRuleTechnicalLemma
\end{lemma}
\begin{proof}
	Since $Z\subseteq C \subseteq XYZ$, then $C=C_XC_YZ$, where $C_X\eqdef X\cap C$, and $C_Y\eqdef Y \cap C$. Also, denote by $X'\eqdef X{\setminus} (C_X\cup A)$, $Y'
	\eqdef Y{\setminus} (C_Y\cup B)$. So, we have that: $I(X;Y|Z)=I(X'C_XA;Y'C_YB|Z)$.
		By the chain rule (see~\eqref{eq:ChainRuleMI}), we have that:
	\begin{align*}
		I(X'C_XA;Y'C_YB|Z)=&I(C_X;C_Y|Z)+I(X'A;C_Y|ZC_X)\\
		&+I(C_X;Y'B|C_YZ)	+\boldsymbol{I(A;B|ZC_XC_Y)}\\
		&+I(A;Y'|ZC_XC_YB)+I(X';Y'B|ZC_XC_YA)
	\end{align*}
	Since $C=ZC_XC_Y$, we get that $I(A;B|C) \leq I(X;Y|Z)$ as required.
\end{proof}

\def\lessVarsLemma{
	Let $\Sigma=\set{\sigma_1,\dots,}$ be a set of triples such that $\vars(\sigma_i)\subseteq \set{X_1,\dots,X_{n-1}}$ for all $\sigma_i \in \Sigma$. Likewise, let $\tau$ be a triple such that $\vars(\tau)\subseteq \set{X_1,\dots,X_{n-1}}$. Then:
	\eat{	If $\Sigma$ and $\tau$ are defined over RVs $\set{X_1,\dots,X_{n-1}}$ then:}
	\begin{align}
		\label{eq:lessVarsLemma}
		\Gamma_n \models_{EI} \Sigma \implies \tau && \text{ if and only if } && \Gamma_{n-1} \models_{EI} \Sigma \implies \tau
	\end{align}
}

\begin{lemma}
	\label{lem:implicationCovers}
	Let $\tau=(A;B|C)$.	If $\Gamma_n \models_{EI} \Sigma \implies \tau$ then there exists a triple $\sigma=(X;Y|Z) \in \Sigma$ such that:
	\begin{enumerate}[noitemsep]
		\item 	$XYZ \supseteq ABC$, and
		\item $ABC \cap X \neq \emptyset$ and $ABC \cap Y \neq \emptyset$.
	\end{enumerate}
\end{lemma}
\begin{proof}
	Let $\tau=(A;B|C)$, where $A=a_1\dots a_m$, $B=b_1\dots b_\ell$, $C=c_1\dots c_k$, and $U\eqdef \Omega{\setminus}ABC$. Following~\cite{DBLP:journals/iandc/GeigerPP91}, we construct the parity  distribution $P(\Omega)$ as follows. 
	We let all the RVs, except $a_1$, be independent binary RVs with probability $\frac{1}{2}$ for each of their two values, and let $a_1$ be determined from $ABC{\setminus }\set{a_1}$ as follows:
	\begin{equation}
		\label{eq:parity}
		a_1 = \sum_{i=2}^m a_i + \sum_{i=1}^\ell b_i + \sum_{i=1}^k c_i \pmod{2}
	\end{equation}
	Let $D \subseteq \Omega$ and $\bd \in \D(D)$. We denote by $D_{ABC}\eqdef D \cap ABC$, and by $\bd_{ABC}$ the assignment $\bd$ restricted to the RVs $D_{ABC}$.
	We show that if $D_{ABC}\subsetneq ABC$ then the RVs in $D$ are mutually independent.\eat{
		Now, let $D \subseteq \Omega$, such that $D \cap \vars(\tau) \subsetneq ABC$. We show that the RVs in $D$ are pairwise independent. Let $\bd \in \D(D)$, and we denote $D{\cap} ABC$ by $D_{ABC}$, and by $\bd_{ABC}$ the assignment $\bd$ restricted to the RVs $D_{ABC}$.}
	By the definition of $P$ we have that:
	\begin{align*}
		P(D=\bd)=\left(\frac{1}{2}\right)^{|D\cap U|}P(D_{ABC}=\bd_{ABC})
	\end{align*}
	There are two cases with respect to $D$. If $a_1\notin D$ then, by definition, $P(D_{ ABC}=\bd_{ABC})=\left(\frac{1}{2}\right)^{|D_{ABC}|}$, and overall we get that $P(D=\bd)=\left(\frac{1}{2}\right)^{|D|}$, proving that the RVs in $D$ are mutually independent.
	If $a_1 \in D$, then since $D_{ABC}\subsetneq ABC$ it holds that $P(a_1|D_{ABC}{\setminus}\set{a_1}){=}P(a_1)$. To see this, observe that:
	\begin{align*}
		P(a_1=1|D_{ABC}{\setminus}\set{a_1}) =\begin{cases}
			\frac{1}{2} & \text{if } \sum_{y\in D_{ABC}{\setminus}\set{a_1}}y \pmod{2}{=}0 \\
			\frac{1}{2} & \text{if }\sum_{y\in D_{ABC}{\setminus}\set{a_1}}y \pmod{2}{=}1 
		\end{cases}
	\end{align*}
	because if, w.l.o.g, $\sum_{y\in D_{ABC}{\setminus}\set{a_1}}y \pmod{2}=0$, then $a_1=1$ implies that $\sum_{y\in ABC{\setminus}D}y \pmod{2}=1$, and this is the case for precisely half of the assignments $ABC{\setminus}D{\rightarrow}\set{0,1}^{|ABC{\setminus}D|}$. 
	Hence, for any $D \subseteq \Omega$ such that $D \cap \vars(\tau) \subsetneq ABC$, it holds that $P(D=\bd)=\prod_{y \in D}P(y=\bd_y)=\left(\frac{1}{2}\right)^{|D|}$, and therefore the RVs in $D$ are mutually independent.
	
	By definition of entropy (see~\eqref{eq:entropy}) we have that $H(X_i)=1$ for every binary RV in $\Omega$. Since the RVs in $D$ are mutually independent, then $H(D)=\sum_{y\in D}H(y)=|D|$ (see Section~\ref{subsec:information:theory}). Furthermore, for any $(X;Y|Z)\in \Sigma$ such that $XYZ \not\supseteq ABC$ we have that:
	\begin{align*}
		I(X;Y|Z)&=H(XZ)+H(YZ)-H(Z)-H(XYZ)\\
		&=|XZ|+|YZ|-|Z|-|XYZ| \\
		&=|X|+|Y|+|Z|-|XYZ| \\
		&=0
	\end{align*}	
	On the other hand, letting $A'{\eqdef}A{\setminus}\set{a_1}$, then by chain rule for entropies (see~\eqref{eq:chainRuleH}), and noting that, by~\eqref{eq:parity}, $ABC{\setminus}a_1 \fd a_1$, then:
	\begin{align*}
		H(\vars(\tau))=H(ABC)&=H(a_1A'BC) \\
		&=H(a_1|A'BC)+H(A'BC)\\
		&=0+|ABC|-1=|ABC|-1.
	\end{align*}
	and thus
	\begin{align}
		I(A;B|C)&=H(AC)+H(BC)-H(C)-H(ABC) \nonumber \\
		&=|AC|+|BC|-|C|-(|ABC|-1) \label{eq:tau1}\\
		&=1 \nonumber
	\end{align}
	In other words, the parity distribution $P$ of~\eqref{eq:parity} has an entropic function $h_P\in \Gamma_n$, such that $h_P(\sigma)=0$ for all $\sigma \in \Sigma$ where $\vars(\sigma)\not\supseteq ABC$, while $h_P(\tau)=1$. Hence, if $\Gamma_n \models \Sigma \implies \tau$, then there must be a triple $\sigma=(X;Y|Z) \in \Sigma$ such that $XYZ \supseteq ABC$. \eat{and we arrive at a contradiction to the assumption $\Gamma_n \models_{EI} \Sigma \implies \tau$.}
	
	Now, suppose that $ABC\subseteq XYZ$ and that $ABC \cap Y=\emptyset$. In other words, $ABC\subseteq XZ$. We denote $X_{ABC}\eqdef X\cap ABC$ and $Z_{ABC}=Z\cap ABC$. Therefore, we can write $I(X;Y|Z)$ as $I(X_{ABC}X';Y|Z_{ABC}Z')$ where $X'=X{\setminus}X_{ABC}$ and $Z'=Z{\setminus}Z_{ABC}$. Since $X_{ABC}Z_{ABC}=ABC$, and due to the properties of the parity distribution, we get: 
	\begin{align*}
		&I(X_{ABC}X';Y|Z_{ABC}Z')=\\
		&H(X_{ABC}Z_{ABC}X'Z')+H(YZ_{ABC}Z')-H(Z_{ABC}Z')-H(X_{ABC}Z_{ABC}X'Z'Y)=\\
		&H(ABCX'Z')+H(YZ_{ABC}Z')-H(Z_{ABC}Z')-H(ABCX'Z'Y)=\\
		&H(ABC)+H(X'Z')+H(Y)+H(Z_{ABC}Z')-H(Z_{ABC}Z')-H(ABC)-H(X'Z'Y)=0
	\end{align*}
Symmetrically, if $ABC \subseteq YZ$, then $I(X;Y|Z)=0$.\eat{
	It is easily shown that if $ABC \subseteq X$ or $ABC \subseteq Z$ then $I(X;Y|Z)=0$.	
	Otherwise (i.e., $X_{ABC}\neq \emptyset$ and $Z_{ABC}\neq \emptyset$), then due to the properties of the parity function, we have that $H(YZ'Z_{ABC})=H(Y)+H(Z')+H(Z_{ABC})$. Noting that $X_{ABC}Z_{ABC}=ABC$, we get that $I(X_{ABC}X';Y|Z_{ABC}Z')=0$.}
	
	Overall, we showed that for all triples $(X;Y|Z)\in \Sigma$ that do not meet the conditions of the lemma, it holds that $I_{h_P}(X;Y|Z)=0$, while $I_{h_P}(A;B|C)=1$ (see~\eqref{eq:tau1}) where $h_P$ is the entropic function associated with the parity function $P$ in~\eqref{eq:parity}. Therefore, there must be a triple $\sigma\in \Sigma$ that meets the conditions of the lemma. Otherwise, we arrive at a contradiction to the assumption that $\Gamma_n \models \Sigma \implies \tau$.	
\end{proof}

\section{Approximate Implication For Saturated CIs}
\label{sec:saturatedCIs}
In this section we prove Theorem~\ref{thm:saturated}. In fact, we prove the following stronger statement.
\begin{theorem}
	\label{thm:newSaturated}
Let $\tau=(A;B|C)$, and let $\Sigma$ be a set of saturated CIs and conditionals. Then:
\begin{equation}
	\Delta_n\models_{EI} \Sigma \implies \tau \text{ if and only if } \Gamma_n \models h(\tau) \leq \min\set{|A|,|B|}\cdot h(\Sigma)
\end{equation}
\end{theorem}
Since $\Delta_n \subset \Gamma_n$, then if $\Gamma_n\models_{EI} \Sigma \implies \tau$, then clearly $\Delta_n\models_{EI} \Sigma \implies \tau$. Theorem~\ref{thm:newSaturated} establishes that if $\Sigma$ is a set of saturated CIs, then $\Delta_n\models_{EI} \Sigma \implies \tau$ implies the (stronger!) statement that  $\Gamma_n\models_{EI} \Sigma \implies \tau$.
Previously, it was shown that if $\Sigma$ is a set of saturated CIs, and $\Delta_n\models_{EI} \Sigma \implies \tau$, then $\Gamma_n \models h(\tau)\leq |A|{\cdot}|B|{\cdot} h(\Sigma)$~\cite{DBLP:journals/lmcs/KenigS22}. In particular, $h(\tau)\leq \frac{n^2}{4}{\cdot} h(\Sigma)$. Since $|A\cup B|\leq n$, then $\min\set{|A|,|B|}\leq \frac{n}{2}$, leading to the significantly tighter bound of $h(\tau)\leq \frac{n}{2}\cdot h(\Sigma)$.
\eat{For a saturated CI $\sigma$, we have that $\vars(\sigma)=\Omega \supseteq \vars(\tau)$. Therefore, Theorem~\ref{thm:newSaturated} immediately applies when $\Sigma$ is a set of saturated CIs.}

Before proceeding, we show that w.l.o.g,  we can assume that $\Sigma$ consists of only saturated CIs. If $\Sigma$ contains a non-saturated term, then it must be a conditional $X\rightarrow Y$. We claim that $h(X\rightarrow Y)=I_h(Y;\Omega{\setminus}XY|X)+I_h(Y;Y|\Omega{\setminus}Y)$. Indeed,
\begin{align}
	&I_h(Y;\Omega{\setminus}XY|X)+I_h(Y;Y|\Omega{\setminus}Y) \nonumber \\
	&=h(XY)+h(\Omega{\setminus}Y)-h(X)-h(\Omega)+2h(\Omega)-h(\Omega{\setminus}Y)-h(\Omega)\nonumber \\
	&=h(XY)-h(X)\eqdef h(Y|X)\eqdef h(X\rightarrow Y) \label{eq:fromCondToSat}
\end{align}
We create a new set of CIs by replacing every conditional $X\rightarrow Y\in \Sigma$ with the two saturated CIs $(Y;\Omega{\setminus}XY|X)$ and $(Y;Y|\Omega{\setminus}Y)$. Denoting by $\Sigma'$ the new set of CIs, it follows from~\eqref{eq:fromCondToSat} that $h(\Sigma)=h(\Sigma')$. Therefore, we assume w.l.o.g that $\Sigma$ contains only saturated CIs.
The proof of Theorem~\ref{thm:newSaturated} relies on the following Lemma, proved in Section~\ref{subsec:newSaturated}.
\begin{lemma}
	\label{lem:neqSaturated}
Let  $\tau=(a;B|C)$ where $a\in \Omega$ is a singleton, and $B,C \subseteq \Omega$, and let $\Sigma$ be a set of saturated CIs. Then:
\begin{equation}
	\Delta_n\models_{EI} \Sigma \implies \tau \text{ if and only if } \Gamma_n \models h(\tau) \leq h(\Sigma)
\end{equation}
\end{lemma}

\subsubsection*{Proof of Theorem~\ref{thm:newSaturated}}
We consider two cases: first, where $\tau=(A;B|C)$ is a CI where $A$ and $B$ are disjoint, and second where $\tau=(C\rightarrow A)=(A;A|C)$.
Let $\tau=(A;B|C)$ be a CI where $A\cap B=\emptyset$, and suppose, without loss of generality, that $|A|\leq |B|$, and that $A= a_1\dots a_K$. By the chain rule of mutual information (see~\eqref{eq:ChainRuleMI}), we have that:
\[
h(\tau)=I_h(a_1\dots a_K;B|C)=I_h(a_1;B|C)+I_h(a_2;B|a_1C)+\cdots +I_h(a_K;B|a_1\dots a_{K-1}C)
\]
By Lemma~\ref{lem:neqSaturated}, we have that $I_h(a_i;B|a_1\dots a_{i-1} C)\leq h(\Sigma)$ for every $i\in \set{1,\dots,K}$. Hence, we get that $h(\tau)\leq |A|\cdot h(\Sigma)$. 

Now, let $\tau=(C\rightarrow A)=(A;A|C)$, where $A=a_1\dots a_K$. By the chain rule for entropy (see~\eqref{eq:chainRuleH}), we have that:
\begin{align*}
h(C\rightarrow A)&=h(a_1\dots a_K|C)=h(a_1|C)+h(a_2|a_1C)+\cdots +h(a_K|a_1\dots a_{K-1}C)\\
&=I_h(a_1;a_1|C)+I_h(a_2;a_2|a_1C)+\cdots + I_h(a_K;a_K|a_1\dots a_{K-1}C)
\end{align*}
By Lemma~\ref{lem:neqSaturated}, we have that $I_h(a_i;a_i|a_1\dots a_{i-1} C)\leq h(\Sigma)$ for every $i\in \set{1,\dots,K}$. Hence, we get that $h(\tau)\leq |A|\cdot h(\Sigma)$. This completes the proof.

\subsection{Proof of Lemma~\ref{lem:neqSaturated}}
\label{subsec:newSaturated}
If  $\Gamma_n \models h(\tau) \leq h(\Sigma)$, then whenever $h(\Sigma)=0$, it holds that $h(\tau)\leq h(\Sigma)=0$. By the non-negativity of the Shannon information measures, we have that $\Gamma_n \models 0\leq h(\tau)$. Therefore, if $h(\Sigma)=0$, then $h(\tau)=0$; or that $\Gamma_n \models_{EI} \Sigma \implies \tau$. Since $\Delta_n\subseteq \Gamma_n$, then $\Delta_n \models_{EI} \Sigma \implies \tau$. 

Now, suppose that $\Delta_n \models_{EI} \Sigma \implies \tau$, and that $|\Omega|=n$.  We prove the claim by reverse induction on $|C|$.
\paragraph*{Base.} There are two base cases, where $|C|=n-1$, and $|C|=n-2$. If $|C|=n-1$, then  $\tau=(a;a|C)$ where $a\in \Omega$. That is, $\tau=(C\rightarrow a)$ is a conditional, and $aC=\Omega$. By Lemma~\ref{lem:simpleExistenceLemma}, there exists a CI $\sigma=(X;Y|Z)\in \Sigma$ such that $Z\subseteq C$, $X\not\subseteq C$, and $Y\not\subseteq C$. Since $X\not\subseteq C$, then $X\cap (\Omega{\setminus}C)=X\cap \set{a} \neq \emptyset$, and hence $a\in X$. Likewise, we get that $a\in Y$. We denote by $X_C\eqdef X\cap C$, and $Y_C\eqdef Y\cap C$. By Lemma~\ref{lem:chainRuleTechnicalLemma}, we have that:
\begin{equation*}
	h(\sigma)=I_h(aX_C ; aY_C|Z)\underbrace{\geq}_{\text{Lemma}~\ref{lem:chainRuleTechnicalLemma}} I_h(a;a|ZX_CY_C)\underbrace{=}_{C\subseteq XYZ}I_h(a;a|C)=h(\tau)
\end{equation*}
Since $\sigma \in \Sigma$, we get that $\Gamma_n\models h(\tau)\leq h(\Sigma)$.

Now, assume that $|C|=n-2$.
This means that $\tau=(a;b|C)$ where $a,b\in \Omega$ are singletons, and that $|abC|=n$. Therefore, $abC=\Omega$. By Lemma~\ref{lem:simpleExistenceLemma}, it holds that there exists a CI $\sigma=(X;Y|Z)\in \Sigma$ such that $Z\subseteq C$, $X\not\subseteq C$, and $Y\not\subseteq C$. 
Since $X\not\subseteq C$, then $X\cap (\Omega{\setminus}C)\neq \emptyset$, and hence $X\cap ab\neq \emptyset$. Similarly, since $Y\not\subseteq C$, then $Y\cap ab \neq \emptyset$. Since $Z\subseteq C$, then $ab\cap Z=\emptyset$, and since $ab\subseteq XYZ$, then $ab \subseteq XY$. So, we may assume wlog that $a\in X$ and $b\in Y$. We denote by $X_C\eqdef X\cap C$, and $Y_C\eqdef Y\cap C$. By Lemma~\ref{lem:chainRuleTechnicalLemma}, we have that:
\begin{equation*}
h(\sigma)=I_h(aX_C ; bY_C|Z)\underbrace{\geq}_{\text{Lemma}~\ref{lem:chainRuleTechnicalLemma}} I_h(a;b|ZX_CY_C)\underbrace{=}_{C\subseteq XYZ}I_h(a;b|C)=h(\tau)
\end{equation*}
Since $\sigma \in \Sigma$, we get that $\Gamma_n\models h(\tau)\leq h(\Sigma)$.

\paragraph*{Step.} We now assume that the claim holds for all $C\subseteq \Omega$, where $|C|\geq k+1$, and we prove the claim for the case where $|C|=k$. By Lemma~\ref{lem:simpleExistenceLemma}, it holds that there exists a CI $\sigma=(X;Y|Z)\in \Sigma$ such that $Z\subseteq C$, $X\not\subseteq C$, and $Y\not\subseteq C$. Since $XYZ \supseteq aBC$, and $Z\subseteq C$, then $XY\supseteq aB$.
So, we may assume wlog that $a\in X$. As before, we denote by $X_C\eqdef X\cap C$, $X_B\eqdef X\cap B$, $X'\eqdef X\setminus aX_BX_C$, $Y_C\eqdef Y\cap C$, $Y_B\eqdef Y\cap B$, and $Y'\eqdef Y\setminus Y_CY_B$. We write $h(\sigma)=I_h(aX_BX_CX';Y_BY_CY'|Z)$ as follows.
\begin{align}
	I_h(aX_BX_CX';Y_BY_CY'|Z)&=I_h\underbrace{(aX_BX_CX';Y_C|Z)}_{\eqdef \alpha_1}+I_h(aX_BX_CX';Y_BY'|Y_CZ)\\
	&=h(\alpha_1)+I_h\underbrace{(aX_BX_CX';Y_B|Y_CZ)}_{\eqdef \sigma_1}+I_h\underbrace{(aX_BX_CX';Y'|Y_BY_CZ)}_{\eqdef \sigma_2} \label{eq:partitionsigmatag}
\end{align}
Now, we consider two options for $\sigma$: (1) $Y_B\neq\emptyset$, and (2) $Y_B= \emptyset$.
\eat{
If $Y_B\neq\emptyset$, then, by Lemma~\ref{lem:chainRuleTechnicalLemma}, we have that:
\begin{equation}
	\label{eq:partitionsigma}
	h(\sigma_1)= I_h(aX_BX_CX';Y_B|Y_CZ)\underbrace{=}_{\eqref{eq:ChainRuleMI}} I_h\underbrace{(X_C;Y_B|Y_CZ)}_{\eqdef \alpha_2}+I_h\underbrace{(aX_BX';Y_B|X_CY_CZ)}_{\eqdef \sigma_3}
\end{equation}
where we denote $\alpha_2\eqdef (X_C;Y_B|Y_CZ)$, and $\sigma_3=(aX_BX';Y_B|X_CY_CZ)$.
}

We first consider the case where $Y_B\neq \emptyset$. In that case, we express $h(\tau)$ as follows:
\begin{equation}
	\label{eq:partitiontau}
	h(\tau)=I_h(a;B|C)=I_h(a;X_BY_B|C)\underbrace{=}_{\eqref{eq:ChainRuleMI}} I_h\underbrace{(a;Y_B|C)}_{\tau_1}+I_h\underbrace{(a;X_B|Y_BC)}_{\tau_2}
\end{equation}
Since $\vars(\sigma_1)\supseteq \vars(\tau_1)$, and $Y_CZ\subseteq C$, then by Lemma~\ref{lem:chainRuleTechnicalLemma}, it holds that $h(\tau_1)\leq h(\sigma_1)$.
We define $\Sigma_2\eqdef (\Sigma{\setminus} \set{\sigma})\cup \set{\sigma_2}$. Since $\vars(\sigma_2)=\vars(\sigma)=\Omega$, then $\sigma_2$ is saturated, and hence $\Sigma_2$ is a set of saturated CIs. 
We claim that $\Delta_n \models_{EI} \Sigma_2 \implies \tau_2$. This completes the proof for the case where $Y_B\neq \emptyset$, because $|Y_BC|>|C|=k$. By the induction hypothesis $\Gamma_n \models h(\tau_2)\leq h(\Sigma_2)$. Therefore,
\[
h(\tau)=h(\tau_1)+h(\tau_2)\leq  h(\sigma_1)+h(\Sigma_2)\leq h(\alpha_1) + h(\sigma_1)+h(\Sigma_2) = h(\Sigma)
\]
So, we show that $\Delta_n \models_{EI} \Sigma_2 \implies \tau_2$. Since $\Delta_n \models_{EI} \Sigma \implies \tau$, then by~\eqref{eq:partitiontau}, it holds that $\Delta_n \models_{EI} \Sigma \implies \tau_2$. Observe that $\Sigma=\Sigma_2\cup \set{\sigma_1,\alpha_1}$. Since $Y_B$ belongs to the conditioned part in $\tau_2$ (i.e., $\iset(\tau_2)\subseteq \isetc(Y_B)$), and since $Y_C\subseteq C$, then by Corollary~\ref{corr:excludeSigma}, it holds that $\Delta_n \models_{EI} \Sigma_2 \implies \tau_2$, as required. 
\eat{
By Lemma~\ref{lem:implicationInclusion}, it is enough to show that $m(\tau_2)\subseteq m(\Sigma_2)$. Since $\Delta_n\models_{EI} \Sigma_1\implies \tau$, then by Lemma~\ref{lem:implicationInclusion}, it holds that $m(\tau)\subseteq m(\Sigma_1)=m(\sigma)\cup m(\Sigma_2)$. That is,
\[
m(\tau)\underbrace{=}_{\eqref{eq:partitiontau}}m(\tau_1)\cup m(\tau_2)\subseteq m(\sigma)\cup m(\Sigma_2)\underbrace{=}_{\eqref{eq:partitionsigma}}  m(\sigma_1)\cup m(\sigma_2)\cup m(\Sigma_2)
\]
In other words, we have that $m(\tau_2)\subseteq m(\sigma_1)\cup m(\sigma_2)\cup m(\Sigma_2)$. 
Since $m(\tau_2)\cap m(\sigma_1)=\emptyset$, and $m(\tau_2)\cap m(\sigma_2)=\emptyset$, then $m(\tau_2)\subseteq m(\Sigma_2)$. This proves the claim.
}

We now consider the case where $Y_B= \emptyset$. Since $Z\subseteq C$, and $XYZ\supseteq aBC$, this means that $B\subseteq X$.
In this case, we can express $h(\sigma)$ as follows:
\begin{align}
	h(\sigma)&=I_h(aBX_CX';Y_CY'|Z) = I_h\underbrace{(aBX_CX';Y_C|Z)}_{\eqdef \alpha_1}+I_h(aBX_CX';Y'|Y_CZ)\nonumber \\
	&=h(\alpha_1)+I_h\underbrace{(X_C;Y'|Y_CZ)}_{\eqdef \alpha_2}+ I_h(aBX';Y'|X_CY_CZ)\nonumber \\
	&=h(\alpha_1)+h(\alpha_2)+I_h\underbrace{(a;Y'|C)}_{\eqdef \sigma_1}+I_h\underbrace{(B;Y'|aC)}_{\eqdef \sigma_2}+I_h\underbrace{(X';Y'|aBC)}_{\eqdef \alpha_3}
\end{align}
where $\sigma_1\eqdef (a;Y'|C)$, $\sigma_2\eqdef (B;Y'|aC)$, and $\alpha_3\eqdef (X';Y'|aBC)$. Since $Y\not\subseteq C$, then $Y'\neq \emptyset$.
In particular, we have that $\Delta_n \models_{EI} \Sigma \implies \sigma_2$, and $\Delta_n \models_{EI} \Sigma \implies \tau$. By the chain rule (see~\eqref{eq:ChainRuleMI}), we have that:
\[
h(\tau)+h(\sigma_2)=I_h(a;B|C)+I_h(B;Y'|aC)=I_h\underbrace{(aY';B|C)}_{\tau'}
\]
where we denote $\tau'\eqdef(aY';B|C)$. Hence, it holds that $\Delta_n \models_{EI} \Sigma \implies \tau'$. We express $h(\tau')$ as follows:
\begin{equation*}
	h(\tau')=I_h(aY';B|C)=I_h\underbrace{(B;Y'|C)}_{\tau_1}+I_h\underbrace{(a;B|Y'C)}_{\tau_2}
\end{equation*}
where we denote $\tau_1\eqdef (B;Y'|C)$, and $\tau_2\eqdef (a;B|Y'C)$. Since $\Delta_n \models_{EI} \Sigma \implies \tau'$, then $\Delta_n \models_{EI} \Sigma \implies \tau_2$. Since $Y_CY'\subseteq Y'C$, then by Corollary~\ref{corr:excludeSigma}, it holds that $\Delta_n \models_{EI} \Sigma{\setminus}\set{\sigma} \implies \tau_2$. Letting $\Sigma_2\eqdef \Sigma{\setminus} \set{\sigma}$, we get that $\Delta_n \models_{EI} \Sigma_2 \implies \tau_2$. 

Since $\Sigma_2\subseteq \Sigma$, then all CIs in $\Sigma_2$ are saturated. Since $\tau_2=(a;B|Y'C)$ where $Y'\neq \emptyset$, then $|Y'C|>|C|=k$. Hence, by the induction hypothesis, we have that $\Gamma_n \models h(\tau_2)\leq h(\Sigma_2)$. Also, by Lemma~\ref{lem:chainRuleTechnicalLemma}, we have that $h(\tau_1)=I_h(B;Y'|C)\leq h(\sigma)$. Overall, we get that:
\[
h(\tau)\leq h(\tau')=h(\tau_1)+h(\tau_2)\leq h(\sigma)+h(\Sigma_2)=h(\Sigma)
\]
This completes the proof.

\section{Approximate Implication for Recursive CIs}
\label{sec:recursiveProof}
We prove Theorem~\ref{thm:recursiveCIs}. Let $\Sigma$ be the recursive basis (see~\eqref{eq:recursiveSet}) over the variable set $\Omega{=}\set{X_1,\dots,X_n}$, and let $G$ be the DAG generated by the recursive set $\Sigma$. Let $\sigma =(A;B|C)$ where $A,B,C \subseteq \Omega$ are pairwise disjoint. We denote by $I_G(A;B|C)$ or $I_G(\sigma)$ the fact that $A$ is $d$-separated from $B$ given $C$ in $G$ (see Section~\ref{sec:BNs}). Theorem~\ref{thm:soundnessCompletenessdSep} establishes that the $d$-separation algorithm is sound and complete for inferring CIs from the recursive basis. This means that  $I_G(A;B|C)$ if and only if $\Gamma_n \models \Sigma \implies (A;B|C)$.
 In our proof, we will make use of the following.
\begin{citedlemmaJAIR}{\cite{DBLP:books/daglib/0066829}}
	\label{lem:weaktransitivitydSep}
Let $X,Y,Z\subseteq \Omega$ be pairwise disjoint, and let $\gamma \in \Omega{\setminus}XYZ$. 
\begin{align*}
\text{If }	I_G(X;Y|Z) \text{ and } I_G(X;Y|Z\gamma) \text{ then }I_G(X;\gamma|Z) \text{ or } I_G(Y;\gamma|Z)
\end{align*}
\end{citedlemmaJAIR}
We prove Theorem~\ref{thm:recursiveCIs} by induction on the highest RV-index mentioned in any triple of $\Sigma$. 
\eat{We prove Theorem~\ref{thm:recursiveCIs} by induction on $n$,
	the number of variables in the distribution $P$ over $\Omega{=}\set{X_1,\dots,X_n}$. }
The claim vacuously holds for $n=1$ (since no conditional independence statements are implied), so we assume correctness when the highest RV-index mentioned in $\Sigma$ is $\leq n-1$, and prove for $n$.

The recursive basis contains $n$ CIs, $\Sigma=\set{\sigma_1,\dots,\sigma_n}$, where $\sigma_i=(X_i;Y_i|Z_i)$,  where $Y_iZ_i=\set{X_1,\dots,X_{i-1}}$. In particular, only $\sigma_n=(X_n;Y_n|Z_n)$ mentions the RV $X_n$, and it is saturated (i.e., $X_nY_nZ_n=\Omega$).
We denote by $\Sigma'\eqdef \Sigma{\setminus} \set{\sigma_n}$, and note that $X_n\notin \vars(\Sigma')$.
The induction hypothesis states that:
\begin{align}	
	\Gamma_{n} \models_{EI} \Sigma' \implies \tau && \text{if and only if} && \Gamma_{n} \models h(\Sigma') \geq h(\tau)  \label{eq:induction}
\end{align}
Now, we consider $\tau=(A;B|C)$. We divide to three cases, and treat each one separately.
\begin{enumerate}[itemsep=0mm]
	\item $X_n \notin ABC$
	\item $X_n \in C$
	\item $X_n \in A$ (or, symmetrically, $X_n\in B$)
\end{enumerate}
\paragraph{Case 1: $X_n \notin ABC$.}
By Theorem~\ref{thm:soundnessCompletenessdSep}, it holds that the $d$-separation criterion is complete with respect to the recursive basis. Therefore, if 
 $\Gamma_n \models_{EI} \Sigma \implies (A;B|C)$, then $I_G(A;B|C)$. That is, $A$ and $B$ are $d$-separated given $C$ in $G$. Let $G'$ be the graph that results from $G$ by removing $X_n$ an all edges adjacent to $X_n$. We claim that $I_{G'}(A;B|C)$. If not, then there is an active trail $P=(a,v_1,\dots,v_k,b)$ in $G'$ between a vertex $a\in A$ and $b\in B$, given $C$.  Since all vertices and edges in $P$ are included in $G$, and since the addition of vertices and edges cannot block a trail (see Section~\ref{sec:BNs}), then $P$ is an active trail given $C$ between $a$ and $b$ in $G$. By the completeness of $d$-separation, this implies that $\Gamma_n \not\models \Sigma \implies \tau$, bringing us to a contradiction. Therefore, it holds that $I_{G'}(A;B|C)$. Since the recursive basis associated with $G'$ is $\Sigma'\eqdef \Sigma {\setminus}\set{\sigma_n}$, and since the $d$-separation algorithm is sound, we get that $\Gamma_n \models \Sigma'\implies \tau$. Since $X_n \notin \vars(\Sigma')$, then by the induction hypothesis we get that $\Gamma_n \models h(\Sigma')\geq h(\tau)$.

\paragraph{Case 2: $\tau=(AX_n;B|C)$.}
Recall that $\sigma_n=(X_n;Y_n|Z_n$). We claim that $B\subseteq Y_n$. Suppose otherwise, and let $b\in B{\setminus} Y_n$. Since $\sigma_n$ is saturated, and $b\notin \set{X_n}\cup Y_n$, then $b\in Z_n$. Consider the atom $t\eqdef \iset(X_n)\cap \iset(b)\cap \isetc(\Omega{\setminus}bX_n)$. Clearly, $t\in \iset(\tau)$. On the other hand, $t\notin \iset(\sigma_n)$ because $Y_n \subseteq \Omega{\setminus}bX_n$. For every $\sigma=(X_i;Y_i|Z_i) \in \Sigma'$ either $X_i\in \Omega{\setminus}X_nb$, or $Y_i\subseteq \Omega{\setminus}X_nb$, and hence $t\notin \iset(\sigma)$. Consequently, $\iset(\tau)\not\subseteq \iset(\Sigma')\cup \iset(\sigma_n)=\iset(\Sigma)$. From Corollary~\ref{corr:inclusionGamman}, we get that $\Gamma_n \not\models \Sigma \implies \tau$, which brings us to a contradiction. Therefore, $B\subseteq Y_n$.

Since $\sigma_n$ is saturated, then $ABC \subseteq X_nY_nZ_n$. We denote by $A_Y\eqdef A\cap Y_n$, $A_Z\eqdef A\cap Z_n$, and $Y'\eqdef Y_n{\setminus}ABC$. Similarly, we define $C_Y$, $C_Z$, and $Z'$. Since $B\subseteq Y_n$, we can express $h(\sigma_n)$:
\begin{align*}
	h(\sigma_n)=I_h(X_n;BA_YC_YY'|A_ZC_ZZ')&\geq I_h(X_n;BY'|A_YC_YA_ZC_ZZ')\\
	&=I_h(X_n;BY'|ACZ')\\
	&\geq I_h(X_n;B|ACZ')
\end{align*}
By the chain rule (see~\eqref{eq:ChainRuleMI}), we express $h(\tau)$:
\begin{equation}
	\label{eq:recursiveSetCase3}
	h(\tau)=I_h(AX_n;B|C)=I_h\underbrace{(A;B|C)}_{\tau_1}+I_h\underbrace{(X_n;B|AC)}_{\tau_2}
\end{equation}
Since $\Gamma_n\models_{EI}\Sigma\implies\tau$ then $\Gamma_n\models_{EI}\Sigma\implies\tau_1$, and $\Gamma_n\models_{EI}\Sigma\implies\tau_2$.

We claim that $\Gamma_n \models \Sigma \implies (B;Z'|AC)$. Suppose, by way of contradiction that $\Gamma_n \not\models \Sigma \implies (B;Z'|AC)$. By the soundness of the $d$-separation algorithm (Theorem~\ref{thm:soundnessCompletenessdSep}), there is an active trail $P$ from $b\in B$ to $z\in Z'$ given $AC$ in $G$. By construction, $G$ contains an edge $(z\rightarrow X_n)$ for every $z\in Z_n$. Since $Z'\subseteq Z_n$, then $G$ contains the edge $(z\rightarrow X_n)$, where $z\in Z'$. Therefore, the trail $P$ can be augmented with the edge $(z\rightarrow X_n)$ to form an active trail from $b$ to $X_n$ (given $AC$). Since the $d$-separation algorithm is complete, this means that $\Gamma_n \not\models \Sigma \implies (X_n;B|AC)$. But then, $\Gamma_n \not\models \Sigma \implies \tau$, which brings us to a contradiction. Therefore, $\Gamma_n \models \Sigma \implies (B;Z'|AC)$. Therefore, we have that $\Gamma_n \models \Sigma \implies (B;Z'|AC), (A;B|C)$. By the chain rule, we get that $\Gamma_n \models \Sigma \implies (AZ';B|C)$. Since $X_n \notin ABCZ'$, then by Case 1, we have that $\Gamma_n \models \Sigma' \implies (AZ';B|C)$, and by the induction hypothesis that $\Gamma_n \models h(\Sigma')\geq I_h(AZ';B|C)$. 

Since $\Gamma_n \models \Sigma' \implies (AZ';B|C)$, and $\Gamma_n \models h(\sigma_n)\geq I_h(X_n;B|ACZ')$, we get that:
\begin{align*}
h(\tau)=I_h(AX_n;B|C)\leq I(AX_nZ';B|C)&=I_h(AZ';B|C)+I_h(X_n;B|ACZ')\\
&\leq h(\Sigma')+h(\sigma_n)\\
&=h(\Sigma)
\end{align*}

\paragraph{Case 3: $\tau=(A;B|CX_n)$.}
We claim that $\Gamma_n \models_{EI} \Sigma \implies (A;B|C)$. Suppose, by way of contradiction, that $\Gamma_n \not\models_{EI} \Sigma \implies (A;B|C)$. By the soundness of the $d$-separation algorithm, it holds that $A$ and $B$ are not $d$-separated, given $C$, in the DAG $G$. By construction, $X_n$ is a sink vertex in $G$ (i.e., it has only incoming edges). Consequently, this means that $A$ and $B$ are not $d$-separated given $CX_n$. But then, by the completeness of the $d$-separation algorithm it holds that $\Gamma_n \not\models_{EI} \Sigma \implies (A;B|CX_n)$, which brings us to a contradiction. Therefore, it holds that $\Gamma_n \models_{EI} \Sigma \implies (A;B|C),(A;B|CX_n)$. By Lemma~\ref{lem:weaktransitivitydSep}, this means that $\Gamma_n \models_{EI} \Sigma \implies (A;X_n|C)$, or $\Gamma_n \models_{EI} \Sigma \implies (B;X_n|C)$. We divide to cases accordingly.
If $\Gamma_n \models_{EI} \Sigma \implies (A;X_n|C)$, then by the chain rule, we have that:
\begin{align}
\Gamma_n \models_{EI} \Sigma \implies (A;X_n|C),(A;B|CX_n) \Longrightarrow \Gamma_n \models_{EI} \Sigma \implies (A;BX_n|C) \label{eq:weakTransitivityCase1}
\end{align}
If $\Gamma_n \models_{EI} \Sigma \implies (B;X_n|C)$, then by the chain rule, we have that:
\begin{align}
	\Gamma_n \models_{EI} \Sigma \implies (B;X_n|C),(A;B|CX_n) \Longrightarrow \Gamma_n \models_{EI} \Sigma \implies (AX_n;B|C) \label{eq:weakTransitivityCase2}
\end{align}
Both cases~\eqref{eq:weakTransitivityCase1} and~\eqref{eq:weakTransitivityCase2}, bring us back to case 2. Hence, we have that:
\begin{align*}
	&h(\Sigma)\underbrace{\geq}_{\eqref{eq:weakTransitivityCase1}} I_h(A;BX_n|C) \underbrace{\geq}_{\text{Lemma~\ref{lem:chainRuleTechnicalLemma}}} I_h(A;B|CX_n) && \text{ if~\eqref{eq:weakTransitivityCase1} holds}&& \\
	&h(\Sigma)\underbrace{\geq}_{\eqref{eq:weakTransitivityCase2}} I_h(AX_n;B|C) \underbrace{\geq}_{\text{Lemma~\ref{lem:chainRuleTechnicalLemma}}} I_h(A;B|CX_n) && \text{ if~\eqref{eq:weakTransitivityCase2} holds}&&
\end{align*}
So, in both cases we get that $\Gamma_n \models_{EI} h(\Sigma) \geq h(\tau)$ as required.
This completes the proof. 

\subsection{Tightness of Bound}
Consider the probability distribution $P$ over $\Omega=\set{X_1,\dots,X_n}$, and suppose that the following recursive set of CIs holds in $P$:
\begin{equation}
	\Sigma = \set{(X_1;X_i|X_2\dots X_{i-1}): i\in \set{2,\dots,n}}
\end{equation}
Let $\tau=(X_1;X_2X_3\dots X_n)$. It is not hard to see that by the chain rule:
\begin{equation}
	\label{eq:tightBound}
	I(X_1;X_2X_3\dots X_n)=\sum_{i=2}^nI(X_1;X_i|X_2\dots X_{i-1})=h(\Sigma)
\end{equation}
Hence, $\Gamma_{n} \models_{EI} \Sigma \implies \tau$, and the bound of~\eqref{eq:tightBound} is tight.
\section{Approximate Implication for Marginal CIs}
\label{sec:marginalProof}
In this section, we prove Theorem~\ref{thm:marginal}.
Let $\Sigma$ be a set of marginal mutual information terms, and let $\tau=(A;B|C)$ such that $\Gamma_n{\models_{EI}}\Sigma{\implies}\tau$. By the chain rule (see~\eqref{eq:ChainRuleMI}), we can write $\tau=(a_1\dots a_K;B|C)$ as follows:
\begin{equation}
	h(\tau)=I_h(a_1\dots a_K;B|C)=I_h(a_1;B|C)+,\dots,+I_h(a_K;B|Ca_1\dots a_{K-1})
\end{equation}
We show, in Theorem~\ref{thm:marginalmain}, that if $\Gamma_n\models_{EI} \Sigma \implies (a_i;B|C a_1\dots a_{i-1})$, then $\Gamma_n \models h(\Sigma)\geq I_h(a_i;B|C a_1\dots a_{i-1})$, and thus $\Gamma_n \models h(\Sigma) \geq \min\set{|A|,|B|}h(\tau)$, as required.

Let $\Sigma$ be a set of marginal CIs defined over variables $\Omega$, and let $U\subseteq \Omega$. We denote by $\Sigma(U)$ the set of CIs projected onto the random variables $U$. Formally:
\begin{equation}
	\Sigma(U)\eqdef \set{(X';Y') : (X,Y) \in \Sigma, X'=X\cap U, Y'=Y\cap U}
\end{equation}
\begin{example}
	Suppose that $\Sigma=\set{(abc;e), (def;ac)}$, then $\Sigma(eac)=\set{(ac;e)}$, while $\Sigma(def)=\emptyset$.
\end{example}
\begin{lemma}
	\label{lem:marginalimplieselemental}
	Let $\Sigma$ be a set of marginal mutual information terms, and let $\tau=(a;b|C)$ be an elemental mutual information term where $a,b\in \Omega$, and $C\subseteq \Omega$. The following holds:
 	\begin{equation}
 		\Gamma_n\models_{EI} \Sigma \implies \tau \text{ if and only if } \Gamma_n\models_{EI} \Sigma(\vars(\tau)) \implies \tau 
 	\end{equation}
\end{lemma}
\begin{proof}
	Since $\Sigma$ is a set of marginal CIs, then by Lemma~\ref{lem:chainRuleTechnicalLemma}, it holds that $h(\Sigma)\geq h(\Sigma(\vars(\tau)))$. Therefore, if $\Gamma_n\models_{EI} \Sigma(\vars(\tau)) \implies \tau$, then clearly $\Gamma_n\models_{EI} \Sigma \implies \tau $.
	
	We prove the other direction by induction on $|C|$. When $|C|=0$, then $\tau=(a;b)$. 
	By Lemma~\ref{lem:implicationCovers}, it holds that there exists a CI $\sigma=(X;Y)$ such that (1) $XY\supseteq ab$, and (2) $ab\cap X\neq \emptyset$ and $ab\cap Y\neq \emptyset$.	In other words, $\sigma=(aX';bY')$, where $X'\eqdef X{\setminus} \set{a}$ and $Y'\eqdef Y{\setminus} \set{b}$. Since $\sigma(\vars(\tau))=(a;b)$, we get that $\Gamma_n\models_{EI} \sigma(\vars(\tau))\implies \tau$, and hence $\Gamma_n\models_{EI} \Sigma(\vars(\tau)) \implies \tau$.
	This proves the lemma for the case where $|C|=0$.
	
	So, we assume correctness for elemental terms $(a;b|C)$ where $|C|{\leq} k$, and prove for $|C|=k+1$.
	Since $\Gamma_n\models_{EI} \Sigma \implies \tau$, then by Lemma~\ref{lem:implicationCovers} there exists a mutual information term $\sigma=(X;Y) \in \Sigma$ such that $XY \supseteq abC$. Hence, we denote  $C=C_XC_Y$, where $C_X\eqdef X{\cap}C$ and $C_Y\eqdef Y{\cap}C$. We also denote $X'\eqdef X\setminus abC$, and $Y'\eqdef Y\setminus abC$.
	There are two cases.
	If $\sigma=(aC_XX';bC_YY')$, then $\sigma(\vars(\tau))=(aC_X;bC_Y)$. By Lemma~\ref{lem:chainRuleTechnicalLemma}, it holds that $I_h(aC_X;bC_Y)\geq I_h(a;b|C)=h(\tau)$. Therefore, $\Gamma_n\models_{EI} \Sigma(\vars(\tau))\implies \tau$.

	Otherwise, w.l.o.g, $\sigma=(abC_XX';C_YY')$.	By item 2 of Lemma~\ref{lem:implicationCovers}, it holds that $C_Y \neq \emptyset$.
	We define:
	\begin{align}
		\alpha_1\eqdef(a;C_Y|C_X) && \alpha_2\eqdef(a;C_Y|bC_X)
	\end{align}	
	By Lemma~\ref{lem:chainRuleTechnicalLemma}, we have that $h(\sigma)\geq h(\alpha_1)$ and $h(\sigma)\geq h(\alpha_2)$, and thus $\Gamma_n\models_{EI}\Sigma\implies \set{\alpha_1,\alpha_2}$.
	Noting that $\tau = (a;b|C_XC_Y)$, we have that $\Gamma_n\models_{EI}\Sigma\implies (a;b|C_XC_Y)$. By the chain rule (see~\eqref{eq:ChainRuleMI}) we have that $\Sigma$ implies:
	$$
	(a;C_Y|C_X), (a;b|C_XC_Y) \implies (a;bC_Y|C_X) \implies \underbrace{(a;b|C_X)}_{\eqdef \tau_1}$$
	In other words, we have that $\Gamma_n \models_{EI} \Sigma \implies (a;b|C_X)$.
	We have established that $C_Y\neq \emptyset$, and thus $C_X{\subsetneq} C$. Therefore, $|C_X|<|C_XC_Y|=|C|=k+1$, and thus $|C_X|\leq k$. By the induction hypothesis, it holds that $\Gamma_n \models_{EI} \Sigma(abC_X)\implies \tau_1$. In particular, this means that $\Gamma_n\models_{EI} \Sigma{\setminus} \set{\sigma} \implies \tau_1$ because $C_YY'\cap abC_X=\emptyset$. Denoting $\Sigma_1 \eqdef \Sigma{\setminus} \set{\sigma}$, we have that:
	\begin{align*}
		\Gamma_n \models_{EI} \Sigma_1(\vars(\tau_1)) \implies \tau_1 && \text{ and } && \Gamma_n \models_{EI} \sigma(\vars(\tau)) \implies \alpha_2
	\end{align*}
By the chain rule (see~\eqref{eq:ChainRuleMI}), this means that 
\begin{align}
	\label{eq:exactMarginal}
	\Gamma_n \models_{EI} \Sigma_1(\vars(\tau_1))\cup \sigma(\vars(\tau)) \implies (a;bC_Y|C_X) \implies (a;b|C_XC_Y)=(a;b|C)=\tau
\end{align}
	Since $\vars(\tau_1)=abC_X \subset \vars(\tau)$, then by Lemma~\ref{lem:chainRuleTechnicalLemma} it holds that $h(\Sigma_1(\vars(\tau_1))) \leq h(\Sigma_1(\vars(\tau)))$. Therefore, we have that if $\Gamma_n \models_{EI} \Sigma_1(\vars(\tau_1)) \implies \tau_1$, then $\Gamma_n \models_{EI} \Sigma_1(\vars(\tau)) \implies \tau_1$. Since $\Gamma_n \models_{EI} \sigma(\vars(\tau)) \implies \alpha_2$, then from~\eqref{eq:exactMarginal}, we get that 
	\begin{align*}
	\Gamma_n &\models_{EI} \Sigma_1(\vars(\tau))\cup \sigma(\vars(\tau)) \implies \tau \text{, and therefore }\\
		\Gamma_n &\models_{EI} \Sigma(\vars(\tau))\implies \tau
	\end{align*}
	as required. This completes the proof.	
\end{proof}

\begin{corr}
		\label{corr:marginalimplies}
	Let $\Sigma$ be a set of marginal mutual information terms, and let $\tau=(A;B|C)$. The following holds: 
		\begin{equation}
		\Gamma_n\models_{EI} \Sigma \implies \tau \text{ if and only if } \Gamma_n\models_{EI} \Sigma(\vars(\tau)) \implies \tau 
	\end{equation}
\end{corr}
\begin{proof}
By the chain rule of mutual information, $\Gamma_n\models_{EI} \Sigma \implies \tau$ if and only if $\Gamma_n\models_{EI} \Sigma \implies (a;b|CA'B')$ for every $a\in A$, $b\in B$, $A'\subseteq A{\setminus} \set{a}$, and $B'\subseteq B{\setminus} \set{b}$. By Lemma~\ref{lem:marginalimplieselemental}, this holds if and only if $\Gamma_n\models_{EI} \Sigma(\vars(\tau)) \implies \tau $.  The other direction follows from the fact that $\Gamma_n\models_{EI} h(\Sigma)\geq h(\Sigma(\vars(\tau)))$.
\end{proof}

\begin{theorem}
	\label{thm:marginalmain}
Let $\Sigma$ be a set of marginal CIs, and let $\tau=(a;B|C)$, where $a\in \Omega$, and $BC \subseteq \Omega$. The following holds:
\begin{align}
	\Gamma_n \models_{EI} \Sigma \implies \tau \text{ if and only if } \Gamma_n \models h(\Sigma)\geq h(\tau)
\end{align}
\end{theorem}
\begin{proof}
We make the assumption that $\Gamma_n \not\models_{EI} \Sigma \implies (a;Bc|C{\setminus}\set{c})$, for every $c\in C$. This is without loss of generality because otherwise, we prove the claim for $\tau'\eqdef (a;Bc|C{\setminus}\set{c})$. By Lemma~\ref{lem:chainRuleTechnicalLemma}, it holds that $h(\tau')\geq h(\tau)$. Therefore, if $\Gamma_n \models h(\Sigma)\geq h(\tau')$ then $\Gamma_n \models h(\Sigma)\geq h(\tau)$.

By Corollary~\ref{corr:marginalimplies}, it holds that $\Gamma_n \models_{EI} \Sigma \implies \tau$ if and only if $\Gamma_n \models_{EI} \Sigma(\vars(\tau)) \implies \tau$. So, we prove the claim for $\Sigma(\vars(\tau))$. This gives us the desired result because, by Lemma~\ref{lem:chainRuleTechnicalLemma}, it holds that $\Gamma_n\models h(\Sigma)\geq h(\Sigma(\vars(\tau)))$.
By definition, for every $\sigma \in \Sigma(\vars(\tau))$, it holds that $\vars(\sigma)\subseteq \vars(\tau)=aBC$.

We prove the claim by induction on $|BC|$. If $|BC|=1$, then $\tau=(a;b)$.  By Lemma~\ref{lem:implicationCovers}, there exists a CI $\sigma=(X;Y)\in \Sigma(ab)$, such that $XY\supseteq ab$. Since, $\vars(\Sigma(ab))\subseteq ab$, we get that $ab \subseteq XY \subseteq aBC$, and hence $XY=ab$. Therefore, it must hold that $(X;Y)=(a;b)$, and hence $h(\Sigma)\geq h(\sigma)=h(\tau)$.

So, we assume that the claim holds for $|BC|\leq k$, and prove the claim for the case where $|BC|=k+1$.
Since $\Gamma_n \models_{EI} \Sigma(aBC)\implies \tau$, then by Lemma~\ref{lem:implicationCovers}, there exists a CI $\sigma=(X;Y)\in \Sigma(aBC)$, such that $XY\supseteq aBC$. Since, $\vars(\Sigma(aBC))\subseteq aBC$, we get that $aBC \subseteq XY \subseteq aBC$, and hence $XY=aBC$. We denote by $B_X=B\cap X$, $C_X=C\cap X$, $B_Y=B\cap Y$, and $C_Y=C\cap Y$.
Therefore, there are three options:
First, that $\sigma=(X;Y)=(aBC_X;C_Y)$. In this case, we get that $\Gamma_n\models_{EI} \Sigma\implies (a;C_Y|C_X),(a;B|C_XC_Y)$. By the chain rule, this means that $\Gamma_n\models_{EI} \Sigma\implies (a; BC_Y|C_X)$. But this is a contradiction to our assumption that $\Gamma_n \not\models_{EI} \Sigma \implies (a;Bc|C{\setminus}\set{c})$, for every $c\in C$. 

If $\sigma=(X;Y)=(BC_X;aC_Y)$, or if $\sigma=(X;Y)=(aC_X;BC_Y)$, then the claim clearly follows from Lemma~\ref{lem:chainRuleTechnicalLemma}.
So, we consider the case where $\sigma=(X;Y)=(aB_XC_X;B_YC_Y)$ where $B_X\neq \emptyset$ and $B_Y\neq \emptyset$. Using the chain rule (see~\eqref{eq:ChainRuleMI}) we can write $h(\sigma)=I(aB_XC_X;B_YC_Y)$ as follows:
\begin{align}
I(aB_XC_X;B_YC_Y)&=I(C_Y;aB_XC_X)+I(B_Y;aB_XC_X|C_Y) \nonumber \\
&=I(C_Y;aB_XC_X)+I(B_Y;C_X|C_Y)+I(B_Y;aB_X|C) \nonumber \\
&=I\underbrace{(C_Y;aB_XC_X)}_{\eqdef \sigma_1}+I\underbrace{(B_Y;C_X|C_Y)}_{\eqdef \sigma_2}+I\underbrace{(B_Y;B_X|C)}_{\eqdef \sigma_3}+I\underbrace{(a;B_Y|CB_X)}_{\eqdef\sigma_4}  \label{eq:partitionsigmamarginal}
\end{align}
On the other hand we can write:
\begin{align*}
	h(\tau)=I(a;B|C)=I(a;B_XB_Y|C)=I\underbrace{(a;B_X|C)}_{\eqdef \tau_1}+I(a;B_Y|CB_X)
\end{align*}
Since $\Gamma_n \models_{EI} \Sigma \implies \tau$, then $\Gamma_n \models_{EI} \Sigma \implies \tau_1$. By Corollary~\ref{corr:marginalimplies}, we have that $\Gamma_n \models_{EI} \Sigma(aB_XC) \implies \tau_1$. Let $\Sigma_1 \eqdef \Sigma {\setminus}\set{\sigma} \cup \set{\sigma_1}$. Since $\sigma_1$ is a marginal CI, then $\Sigma_1$ is a set of marginal CIs. Since $B_Y\cap aB_XC =\emptyset$, then from~\eqref{eq:partitionsigmamarginal}, we get that $\Gamma_n \models_{EI} \Sigma_1 \implies \tau_1$. 
Since $B_Y\neq \emptyset$, then $|B_XC|< |BC|=k+1$, and hence $|B_XC|\leq k$. Therefore, by the induction hypothesis, we get that $\Gamma_n \models h(\tau_1)\leq h(\Sigma_1)$. Now, we get that:
\[
h(\tau)=h(\tau_1)+h(\sigma_4) \leq h(\Sigma_1)+h(\sigma_4) \leq h(\Sigma_1)+h(\sigma_2)+h(\sigma_3)+ h(\sigma_4)=h(\Sigma)
\]
This completes the proof.
\end{proof}
\section{Conclusion and Future Work}
\label{sec:conclusion}
We consider the problem of approximate implication for conditional independence. In the general case, approximate implication does not hold~\cite{DBLP:journals/lmcs/KenigS22}. Therefore, we
 establish results and approximation guarantees under various restrictions to the derivation rules, and antecedents (our results are summarized in Table~\ref{tab:AIResults}). The approximation guarantees established in this work have practical implications on algorithms that learn the structure of PGMs from data, and where the CIs used for the generation of the PGM (e.g., recursive basis for Bayesian networks, and the pairwise conditional independencies in Markov networks) hold only approximately. We establish new and tighter approximation bounds when the set of antecendants are saturated, or marginal. We also prove a negative result showing that
  approximate CIs cannot be inferred from the independence graph associated with a Markov network. As part of future work, we intend to investigate restrictions to probability distributions that allow the intersection axiom to relax.

\vskip 0.2in
\bibliography{main_arxiv}
\bibliographystyle{plainurl}% the mandatory bibstyle
\newpage
\appendix
\section{Missing Proofs from Section~\ref{sec:intersectionAxiomDoesNotRelax}}
\label{sec:IntersectionTechnicalDetailes}
\begin{replemma}{\ref{lem:technicalLemmaEntropy}}
	\technicalLemmaEntropy
\end{replemma}
	\begin{proof}
	The proof is by definition, and we provide the technical details.
	For $i\in \set{1,2,3}$:
	\[
	H(A_i)=2\cdot\frac{1}{2}\log 2= 1
	\]
	By the definition of $A_4$ in Table~\ref{tab:A456}, we have that:
	\[
	H(A_4)=-\left((1-3x)\log(1-3x)+3x\log x\right)=\delta_1(x).
	\]
	Proof is the same for $H(A_5)$ and $H(A_6)$.
	
	We now compute $H(A_{4,1})$. From Table~\ref{tab:A456}, we have that $P(A_{4,1}=0)=1-3x+x=1-2x$, and $P(A_{4,1}=1)=2x$. Therefore,
	\[
	H(A_{4,1})=-\left((1-2x)\log(1-2x)+2x\log 2x\right)=\delta_2(x).
	\]
	For symmetry reasons, the same holds for $H(A_{4,2})$ and $H(A_{i,j})$ for $i\in \set{5,6}$ and $j\in \set{1,2}$.
	
	From Table~\ref{tab:A7}, we have that for all $i\in \set{1,2,3}$:
	\[
	P(A_{7,i}=0)=P(A_{7,i}=1)=\frac{1}{2}
	\]
	Therefore, $H(A_{7,i})=1$ for all $i\in \set{1,2,3}$.
	
	Now, take any $i,j\in \set{1,2,3}$ where $i<j$. Then, from Table~\ref{tab:A7}, we have that:
	\begin{equation*}
		P(A_{7,i}=a,A_{7,j}=b)=
		\begin{cases}
			\frac{1}{2}-2y & \text{if }a=b \\
			2y & \text{ otherwise }
		\end{cases}
	\end{equation*}
	Therefore, we have that:
	\[
	H(A_{7,i},A_{7,j})=-\left(2(\frac{1}{2}-2y)\log(\frac{1}{2}-2y)+2\cdot 2y\log 2y \right)=f_2(y)
	\]
	Finally, by Table~\ref{tab:A7}, we have that:
	\[
	H(A_7)=H(A_{7,1},A_{7,2},A_{7,3})=-\left(2(\frac{1}{2}-3y)\log(\frac{1}{2}-3y)+ 6y\log y\right)=f_1(y)
	\]
\end{proof}

\begin{replemma}{\ref{lem:mainTechnicalLemmaIntersection}}
	\mainTechnicalLemmaIntersection
\end{replemma}
	\begin{proof}
	\begin{align*}
		H(A)&\underbrace{=}_{\eqref{eq:A}}H(A_2,A_{6,1},A_{7,1},A_{5,1})\\
		&\underbrace{=}_{\substack{A_1,\dots,A_7 \text{ are} \\ \text{mutually-independent}}} H(A_2)+H(A_{6,1})+H(A_{7,1})+H(A_{5,1}) \\
		&\underbrace{=}_{\text{Lemma }\ref{lem:technicalLemmaEntropy}} 1+\delta_2(x)+1+\delta_2(x) \\
		&=2+2\delta_2(x)
	\end{align*}
	By symmetry, we have that $H(B)=H(C)=2+2\delta_2(x)$ as well.
	
	We now compute $H(A|C)$.
	\begin{align*}
		H(A|C)&\underbrace{=}_{\eqref{eq:A},\eqref{eq:C}}H(A_2,A_{6,1},A_{7,1},A_{5,1}|A_1,A_{5,2},A_{7,3},A_{4,2})\\
		&\underbrace{=}_{\substack{\text{chain rule} \\ \text{for entropy}}} H(A_{5,1},A_{7,1}|A_1,A_{5,2},A_{7,3},A_{4,2})+H(A_2,A_{6,1}|A_1,A_{5,2},A_{7,3},A_{4,2},A_{5,1},A_{7,1})\\
		&\underbrace{=}_{\text{independence}}H(A_{5,1},A_{5,2},A_{7,1},A_{7,3},A_1,A_{4,2})-H(A_1,A_{5,2},A_{7,3},A_{4,2})+H(A_2,A_{6,1}) \\
		&=H(A_{5,1},A_{5,2})+H(A_{7,1},A_{7,3})+H(A_1)+ H(A_{4,2})\\
		&~~~~~-H(A_1)-H(A_{5,2})-H(A_{7,3})-H(A_{4,2})+H(A_2)+H(A_{6,1})\\
		&=H(A_{5,1},A_{5,2})+H(A_{7,1},A_{7,3})-H(A_{5,2})-H(A_{7,3})+H(A_2)+H(A_{6,1})\\
		&\underbrace{=}_{\text{Lemma }~\ref{lem:technicalLemmaEntropy}} \delta_1(x)+f_2(y)-\delta_2(x)-1+1+\delta_2(x)\\
		&=\delta_1(x)+f_2(y)
	\end{align*}
	Now, since $H(AC)=H(A|C)+H(C)$, then by the above, and the fact that $H(C)=2+2\delta_2(x)$, we get that:
	\[
	H(AC)=H(A|C)+H(C)=\delta_1(x)+f_2(y)+2+2\delta_2(x)
	\]
	as required.
	
	Finally, we compute $H(A|BC)$.
	\begin{align*}
		H(A|BC)&\underbrace{=}_{\eqref{eq:A},\eqref{eq:B},\eqref{eq:C}} H(A_2,A_{6,1},A_{7,1},A_{5,1}|A_1,A_{5,2},A_{7,3},A_{4,2},A_3,A_{6,2},A_{7,2},A_{4,1})\\
		&=H(A_1, A_2, A_3,(A_{4,1},A_{4,2}),(A_{6,1},A_{6,2}), (A_{7,1},A_{7,2},A_{7,3}),(A_{5,1},A_{5,2}))\\
		&~~~~~-H(A_1,A_3,(A_{4,1},A_{4,2}),A_{5,2},A_{6,2},A_{7,2},A_{7,3})\\
		&= H(A_1)+H(A_2)+H(A_3)+H(A_4)+H(A_5)+H(A_6)+H(A_7)\\
		&~~~~~-H(A_1)-H(A_3)-H(A_4)-H(A_{5,2})-H(A_{6,2})-H(A_{7,2},A_{7,3})\\
		&=H(A_2)+H(A_5)+H(A_6)+H(A_7)-H(A_{5,2})-H(A_{6,2})-H(A_{7,2},A_{7,3})\\
		&\underbrace{=}_{\text{Lemma }~\ref{lem:technicalLemmaEntropy}}1+\delta_1(x)+\delta_1(x)+f_1(y)-\delta_2(x)-\delta_2(x)-f_2(y)\\
		&=1+2\delta_1(x)+f_1(y)-f_2(y)-2\delta_2(x)
	\end{align*}
	Since $H(ABC)=H(A|BC)+H(BC)$, we get that:
	\begin{align*}
		H(ABC)&=1+2\delta_1(x)+f_1(y)-f_2(y)-2\delta_2(x)+(\delta_1(x)+f_2(y)+2+2\delta_2(x))\\
		&=3+3\delta_1(x)+f_1(y)
	\end{align*}
\end{proof}
\end{document}